\documentclass{article}

\usepackage{microtype}
\usepackage{graphicx}
\usepackage{subfigure}
\usepackage{booktabs} 

\usepackage{hyperref}

\usepackage[accepted]{icml2025}

\usepackage{amsmath}
\usepackage{amssymb}
\usepackage{mathtools}
\usepackage{amsthm}
\usepackage[capitalize,noabbrev]{cleveref}

\DeclarePairedDelimiter{\norm}{\lVert}{\rVert}
\usepackage{soul,xcolor}
\setstcolor{red}
\usepackage[most]{tcolorbox}
\usepackage{booktabs}
\usepackage{adjustbox}
\usepackage{graphicx}
\usepackage{titletoc}
\usepackage{colortbl}
\usepackage{multirow}
\usepackage{enumitem}
\usepackage{caption}
\usepackage{subcaption}
\usepackage{wrapfig}

\definecolor{rowcolor}{RGB}{185, 229, 232} 
\definecolor{conicgrad}{RGB}{255, 192, 217}
\definecolor{cagrad}{RGB}{212, 235, 248}
\definecolor{myred}{RGB}{225, 0, 0}
\definecolor{myblue}{RGB}{0, 112, 192}

\theoremstyle{plain}
\newtheorem{theorem}{Theorem}[section]
\newtheorem{proposition}[theorem]{Proposition}

\theoremstyle{definition}

\theoremstyle{remark}
\newtheorem{remark}[theorem]{Remark}

\usepackage[textsize=tiny]{todonotes}

\icmltitlerunning{Fantastic Multi-Task Gradient Updates and How to Find Them In a Cone}

\begin{document}

\twocolumn[
\icmltitle{Fantastic Multi-Task Gradient Updates and How to Find Them In a Cone}





\icmlsetsymbol{equal}{*}

\begin{icmlauthorlist}
\icmlauthor{Negar Hassanpour}{yyy}
\icmlauthor{Muhammad Kamran Janjua}{equal,yyy}
\icmlauthor{Kunlin Zhang}{equal,yyy} \\
\icmlauthor{Sepehr Lavasani}{yyy}
\icmlauthor{Xiaowen Zhang}{comp}
\icmlauthor{Chunhua Zhou}{yyy}
\icmlauthor{Chao Gao}{yyy}
\end{icmlauthorlist}

\icmlaffiliation{yyy}{Huawei Technologies, Canada}
\icmlaffiliation{comp}{HiSilicon}

\icmlcorrespondingauthor{Negar Hassanpour}{negar.hassanpour2@huawei.com}
\icmlcorrespondingauthor{Muhammad Kamran Janjua}{muhammad.kamran.janjua@huawei.com}

\icmlkeywords{Machine Learning, ICML}

\vskip 0.3in
]



\printAffiliationsAndNotice{\icmlEqualContribution} 

\begin{abstract}

Balancing competing objectives remains a fundamental challenge in multi-task learning (MTL),
primarily due to conflicting gradients across individual tasks. 
A common solution relies on computing a dynamic gradient update vector that balances competing tasks as optimization progresses.
Building on this idea, we propose \textsc{ConicGrad}, a principled, scalable, and robust MTL approach formulated as a constrained optimization problem. 
Our method introduces an angular constraint to dynamically regulate gradient update directions, 
confining them within a cone centered on the reference gradient of the overall objective. 
By balancing task-specific gradients without over-constraining their direction or magnitude, 
\textsc{ConicGrad} effectively resolves inter-task gradient conflicts. 
Moreover, our framework ensures computational efficiency and scalability to high-dimensional parameter spaces. 
We conduct extensive experiments on standard supervised learning and reinforcement learning MTL benchmarks,
and demonstrate that \textsc{ConicGrad} achieves state-of-the-art performance across diverse tasks.

\end{abstract}

\section{Introduction}
\label{intro}

\begin{figure*}[t]
    \centering
    \includegraphics[width=\linewidth]{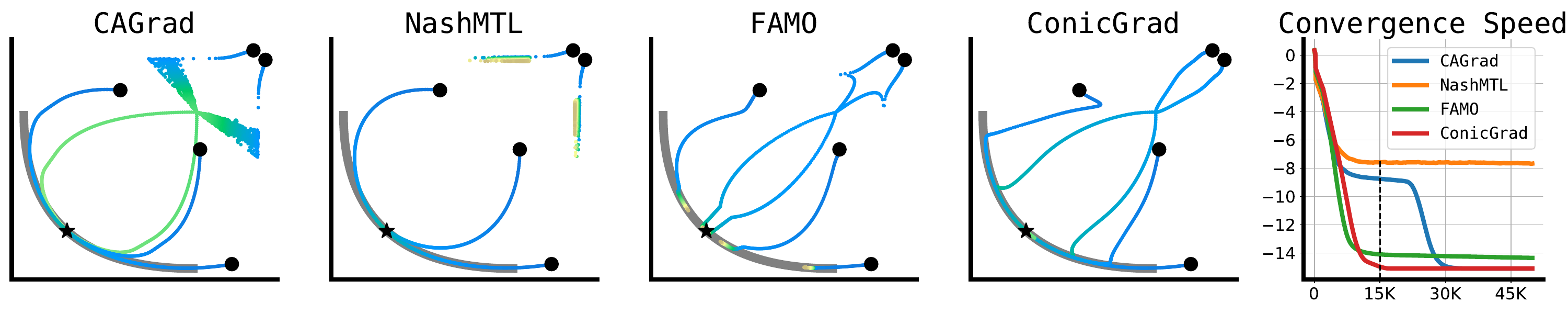}
    \caption{\textbf{Toy Experiment.} 
    The four plots on the left-side visualize the loss trajectories of various MTL methods 
    from 5~initialization points ({\(\bullet\)}) on a toy 2-task learning problem (see \cref{sec:toyres} and \cref{app:toy_detail} for more details). 
    Trajectories transition from blue to green, indicating progress over time. 
    All 5 initialization points for FAMO reach the Pareto front~(gray curve), 
    while and 3 for NashMTL and all 5 for both CAGrad and \textsc{ConicGrad} reach the global minima~(\scalebox{0.7}{\(\bigstar\)})
    with \textsc{ConicGrad} converging significantly faster. 
    The plot on the far-right compares the convergence speeds over training steps, 
    showing that \textsc{ConicGrad} achieves the lowest loss (dashed black line) faster than all competing methods.}
    \label{fig:toyexample}
\end{figure*}

In many real-world machine learning applications, resources such as data, computation, and memory are limited~\citep{navon2022multi,yu2020gradient,xiao2024direction}. 
Therefore, instead of Single-Task Learning (STL)~\citep{long2015fully,he2017mask} that trains an independent model for each downstream task,
it is often advantageous to share parts of the model structure across multiple tasks, a paradigm known as Multi-Task Learning (MTL)~\citep{vandenhende2021multi}.
MTL aims to learn a shared representation while simultaneously optimizing for several different tasks, thereby improving efficiency and generalization. 

MTL approaches can be broadly categorized into multi-task architectures~\citep{lin2025mtmamba,zhang2025sgw} and optimization strategies~\citep{navon2022multi,yu2020gradient,xiao2024direction,NEURIPS2023_b2fe1ee8}. 
Architectural methods leverage parameter sharing to reduce redundancy and maximize learning across tasks. 
However, even with efficient architectures, optimizing multiple losses concurrently remains challenging, 
as naive strategies such as utilizing the reference objective gradient \(g_0\) 
(i.e., uniformly weighted average of all task gradients)
throughout the entire training often lead to sub-optimal performance~\citep{liu2021conflict}.

One of the primary reasons for this challenge is the potential conflicts between task gradients
(i.e., gradients pointing in opposing directions) 
which can impede the concurrent optimization of multiple losses~\citep{yu2020gradient}.
These conflicts often hinder convergence and negatively impact overall performance.
Recent research efforts have focused on optimization strategies that balance task gradients and/or resolve conflicts 
via computing and utilizing a dynamic gradient update vector~\(d\) at each optimization step.

A foundational approach in this direction is Multiple-Gradient Descent Algorithm (MGDA)~\citep{desideri2012multiple},
originally proposed to address Multi-Objective Optimization (MOO).
\citet{sener2018multi} apply MGDA specifically for MTL. 
FAMO~\citep{NEURIPS2023_b2fe1ee8} offers an efficient solution (constant space and time) for the log of MGDA objective. %
However, it may trade off some performance for speed, 
particularly when compared with methods like NashMTL~\citep{navon2022multi} 
and IMTL-G~\citep{liu2021towards}. 
Additionally, approaches such as MGDA and FAMO only guarantee finding Pareto-stationary points%
\footnote{
\label{ft:pareto}
    A Pareto-stationary solution is where no task's loss can be reduced without increasing at least one other task's loss,
    achieved when no descent direction improves all tasks simultaneously.
}
rather than truly optimal solutions.

Other prior works address this by imposing directional constraints on the update vector. 
CAGrad~\cite{liu2021conflict}, for instance, seeks the update vector within a Euclidean ball centered at \(g_0\), 
while SDMGrad~\cite{xiao2024direction} restricts the update direction to be near \(g_0\) through inner-product regularization. 
PCGrad~\citep{yu2020gradient} takes a different approach by directly manipulating the gradients,
projecting one task's gradient onto the normal plane of the others to avoid conflicts. 
Despite their successes, these methods often lack flexibility, as they rely on strict directional constraints or computationally expensive operations, 
or, their stochastic iterative nature limits their applicability in scenarios where a clear, interpretable optimization process is crucial. 

In this work, we propose \textsc{ConicGrad}, a novel multi-task optimization framework designed to address key limitations of existing approaches.
Similar to some prior methods (e.g., \citep{liu2021conflict,xiao2024direction}), 
\textsc{ConicGrad} leverages the reference objective gradient \(g_0\) to ensure alignment with the optimization goal, enabling convergence to an optimum. %
However, unlike these methods, \textsc{ConicGrad} enforces an angular constraint to ensure effective alignment 
without overly restricting either the update direction or magnitude.
Our key contributions are as follows: 
\begin{itemize}[leftmargin=*] 
    \item We formulate \textsc{ConicGrad} and provide a clear geometric interpretation for it, demonstrating how its angular constraint offers greater flexibility compared to the existing, more restrictive methods. 
    \item \textsc{ConicGrad} offers a computationally efficient approach to gradient updates, avoiding the overhead of costly computations or multi-step optimization processes while maintaining high performance.
    \item \textsc{ConicGrad} converges faster (in training steps) than existing methods, demonstrated by benchmark toy example and real-world experiments.
    We also support these empirical findings by theoretical convergence guarantees. 
\end{itemize}

\begin{figure*}[t]
    \centering
    \includegraphics[width=\linewidth]{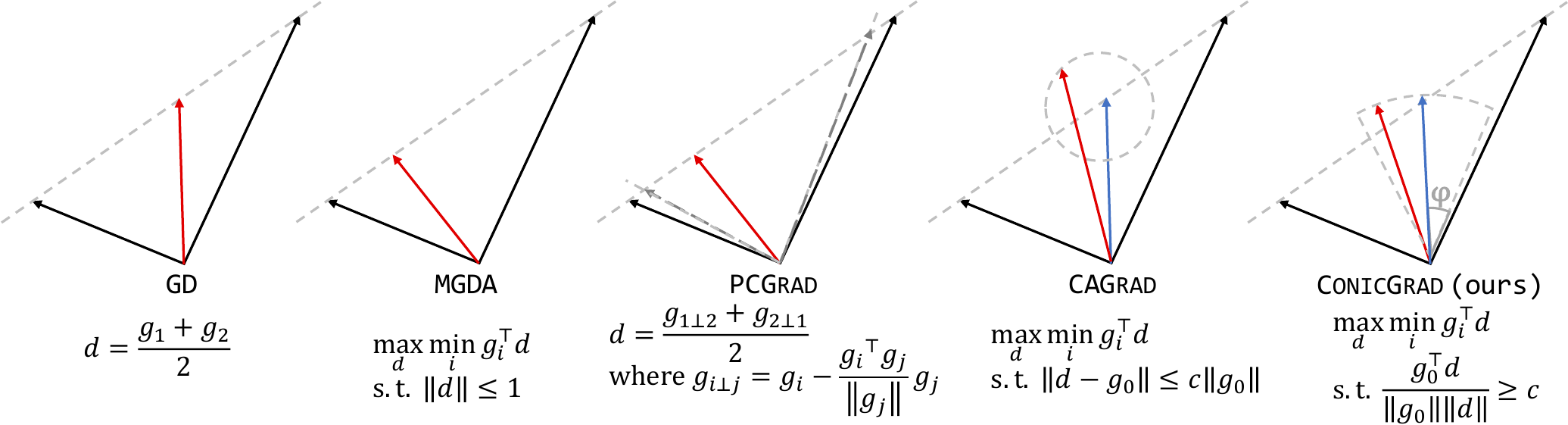}
    \caption{
    \textbf{Visual Illustration of Update Vectors.}
        Inspired by~\citep{liu2021conflict}, we illustrate the update vector \(d\) (in \textcolor{myred}{red}) for a two-task learning problem 
        using various gradient descent methods: GD, MGDA, PCGrad, CAGrad, and \textsc{ConicGrad}.
        Task-specific gradients \(g_1\) and \(g_2\) are in black and the reference objective gradient \(g_0\) is in \textcolor{myblue}{blue}.
        PCGrad projects each gradient onto the plane orthogonal to the other (dashed arrows) and averages the projections.
        CAGrad determines \(d\) by maximizing the minimum improvement across both tasks within a constrained region around the reference gradient \(g_0\).
        \textsc{ConicGrad} determines \(d\) by constraining the update direction to lie within a cone centered around \(g_0\) with an angle at most \(\varphi=\arccos(c)\), 
        ensuring alignment while allowing more flexibility.
    }
    \label{fig:vectors}
\end{figure*}

\section{Preliminaries}
\label{sec:relatedwork}

Consider a multi-task model parameterized by \(\theta \in \mathbb{R}^{M}\) with number of tasks \(K \geq 2\).
Each task has its own objective function, or loss, \(\mathcal{L}_{i}(\theta)\). 
A trivial / reference objective in MTL is optimizing for the uniform average over all the losses, i.e.,

\begin{equation}
    \label{eq:mtleq}
    \theta^{*} = \underset{\theta \in \mathbb{R}^{M}}\arg\min \{\mathcal{L}_{0}(\theta) = \frac{1}{K}\sum_{i=1}^{K}\mathcal{L}_{i}(\theta)\}.
\end{equation}
The parameters can then be updated as \(\theta' \leftarrow \theta - \eta g_0\),
where \(g_0~=~\frac{1}{K} \sum_{i=1}^{K} \nabla_{\theta} \mathcal{L}_{i}\) is the gradient of the reference objective in \cref{eq:mtleq}
and \(\eta\) is the learning rate.
This approach however, is known to be sub-optimal, due to the potential conflicts between task gradients that may occur during training~\citep{liu2021conflict}.

To address these conflicts, we aim to find an alternative update vector \(d\)
that not only decreases the average loss \(\mathcal{L}_{0}\), but also every individual loss.
This can be framed as maximizing the minimum decrease rate across all tasks:
\begin{equation}
\label{eq:mooobj}
    \underset{d \in \mathbb{R}^{M}}\max\underset{i \in [K]}\min \left\{\frac{1}{\eta}(\mathcal{L}_{i}(\theta) - \mathcal{L}_{i}(\theta - \eta d))\right\} \approx \underset{d \in \mathbb{R}^{M}}\max \underset{i \in [K]} \min \langle g_{i},d\rangle,
\end{equation}
where the approximation relies on a first-order Taylor approximation, 
which is accurate when \(\eta\) is small, as is often the case~\cite{liu2021conflict}.

The optimization problem expressed in~\cref{eq:mooobj} can be subjected to various constraints, 
such as \( \norm{d} \leq 1 \) in MGDA~\cite{sener2018multi} and FAMO~\cite{NEURIPS2023_b2fe1ee8}, 
or \( \norm{d-g_0} \leq c \norm{g_0} \) in CAGrad~\cite{liu2021conflict},
which indirectly controls the alignment between \(d\) and \(g_0\) by limiting the deviation in Euclidean space.
As mentioned earlier, 
in the case of the former constraint, the respective algorithms can only reach a Pareto-stationary point,
while for the constraints such as the latter (that incorporate the reference objective gradient \(g_0\)),
there exist guarantees that the algorithm can converge to an optimum.

\section{Multi-Task Learning with \textsc{ConicGrad}}
\label{sec:method}

Our goal is to dynamically compute a gradient update vector \(d\) at each optimization step.
This vector should balance task gradients and mitigate their potential conflicts, 
while ensuring convergence towards the optimum of the reference objective in \cref{eq:mtleq}. 
To this end, we propose \textsc{ConicGrad}, which enforces an angular constraint that restricts \(d\) within a cone centered around \(g_0\). 
Formally, this can be expressed as the following constrained optimization problem
\begin{equation}
\label{eq:org_eq}
    \max\limits_{d \in \mathbb{R}^{M}} \min\limits_{i \in [K]} \langle g_{i},d\rangle \quad s.t. \quad \frac{\langle g_{0}, d\rangle}{\norm{g_0} \norm{d}} \geq c,
\end{equation}
where \(c~\in~[-1, 1]\), and in practice we restrict it to \(c \in (0, 1]\) to avoid negative correlation (see~\cref{appxsec:geometryconic}).

This approach provides an interpretable formulation that maintains sufficient alignment with \(g_0\) 
without imposing overly rigid restrictions on the gradient update.%
\footnote{
    We provide a detailed comparison and interpretation of CAGrad and \textsc{ConicGrad} in \cref{appxsec:conic_vs_cagrad}, 
    highlighting the greater flexibility of \textsc{ConicGrad}'s angular constraint. 
}
Specifically, the advantages include:
(i)~explicit control over the update direction, ensuring that \(d\) remains geometrically aligned with the reference objective gradient \(g_0\); and, 
(ii)~decoupling of magnitude and direction, unlike common directional constraints, which allows \(\norm{d}\) to adopt any task-specific criteria.

To convert~\cref{eq:org_eq} to an unconstrained optimization problem, the Lagrangian in~\cref{eq:org_eq} is
\begin{equation}
\label{eq:primal}
    \max\limits_{d \in \mathbb{R}^{M}} \min\limits_{i \in [K]} \langle g_{i},d\rangle - \frac{\lambda}{2} (c^2 \norm{g_{0}}^2 \norm{d}^2 - \norm{g_0^\top d}^2).
\end{equation}
Note that \(\min\limits_{i \in [K]} \langle g_{i},d\rangle = \min\limits_{\omega \in \mathcal{W}} \langle g_{\omega},d\rangle\), 
where \(g_{\omega} = \sum_{i=1}^K \omega_{i} g_{i}\)
and \(\omega\) is on the probability simplex, i.e., \(\forall i \quad \omega_{i} \geq 0\) and \(\sum_{i=1}^K \omega_{i} = 1\). 
The objective is concave in \(d\), and we find that when \(c < 1\) Slater condition holds (see proof in~\cref{appxprop:slatercond}). We get strong duality since the duality gap is \(0\), and swap the \( \max \) and \( \min \) operators. The dual objective of the primal problem in~\cref{eq:primal} is
\begin{equation}
\label{eq:obj}
    \min_{\substack{\lambda \ge 0 \\ \omega \in \mathcal{W}}} \max\limits_{d \in \mathbb{R}^K} g_\omega^\top d - \frac{\lambda}{2} (c^2 \norm{g_{0}}^2 \norm{d}^2 - \norm{g_0^\top d}^2).
\end{equation}


\begin{proposition}
Given the optimization problem in~\cref{eq:org_eq}, its Lagrangian in~\cref{eq:primal}, and assuming the Slater condition holds, the dual of the primal problem in~\cref{eq:obj}, the optimal gradient update \(d^{*}\) is given by
\begin{equation}
\label{eq:dstar_lmbda}
    d^{*} = \frac{1}{\lambda}\left(c^{2}\norm{g_{0}}^{2}\mathbb{I} - g_{0}g_{0}^\top\right)^{-1} g_{\omega},
\end{equation}
where \(\mathbb{I}\) is a \(M \times M\) identity matrix.
\begin{proof}
Due to space limitations, derivation of the optimal \(d^{*}\) is provided in~\cref{appxderiv:proofsconicgrad}.
\end{proof}
\end{proposition}

\cref{eq:obj} is now simplified to 
\begin{equation}
\label{eq:obj_without_max}
    \min_{\substack{\lambda \ge 0 \\ \omega \in \mathcal{W}}} g_\omega^\top d^{*} - \frac{\lambda}{2} (c^2 \norm{g_{0}}^2 \norm{d^{*}}^2 - \norm{g_0^\top d^{*}}^2).
\end{equation}
We empirically find that \(\lambda~=~1\) works well in practice
(see \cref{app:lambda_opt} for a detailed derivation and discussion), 
and therefore, \cref{eq:dstar_lmbda} is simplified to
\begin{equation}
\label{eq:dstar}
    d^{*} = \big( c^2\norm{g_0}^2\mathbb{I} - g_0 g_0^\top \big)^{-1} g_{\omega}.
\end{equation}
\cref{eq:dstar} provides a closed form expression for computing \(d^{*}\) based on \(g_0\) and \(g_\omega\). 
Since \(d^{*}\) is now analytically found, 
the objective in \cref{eq:obj_without_max} simplifies to an optimization problem only dependent on \(\omega \in \mathcal{W}\)
\begin{equation}
\label{eq:obj2opt}
    \min_{\omega \in \mathcal{W}} g_{w}^\top d^{*} - \frac{1}{2} (c^2 \norm{g_{0}}^2 \norm{d^{*}}^2 - \norm{g_0^\top d^{*}}^2).
\end{equation}

\cref{eq:obj2opt} is a well-shaped optimization problem over \(\omega\),
which can be solved using any standard optimization algorithm. 
Once the optimal \(\omega\) is determined 
(or a step towards it is taken by an iterative optimization method), 
we reuse~\cref{eq:dstar} to derive the final gradient update with the updated \(g_{\omega}\),
and then update the model parameters \(\theta\) accordingly. 

We provide the algorithm for \textsc{ConicGrad} in \cref{alg:algo}.

\paragraph{Normalization for Stability.}  
To ensure that the gradient update remains stable, 
in this work, we scale \(d^{*}\) such that its norm equals that of \(g_{0}\). 
Let \(\Tilde{d}\) denote the final gradient update vector, then we get our final gradient update as
\(
    \Tilde{d}~=~d^{*}\frac{\norm{g_{0}}}{\norm{d^{*}}}
\).
This scaling is merely a design choice.
Since \textsc{ConicGrad}'s constraint decouples direction and magnitude of the update vector,
it accommodates adopting other task-specific criteria/heuristics for such scaling.

\subsection{Efficient Computation of \(d^{*}\)}
\label{sec:smw}
Note that \cref{eq:dstar} requires inverting an \(M\times M\) matrix 
(where \(M\) is the number of model parameters, which can be huge) 
and computing it may be impractical. 
However, due to its specific structure, it is possible to use the Sherman-Morrison-Woodbury (SMW) formula~\citep{higham2002accuracy} 
(also known as the matrix inversion lemma) 
to reformulate it as a \(1\times 1\) (i.e., scalar) inversion instead. 

For the sake of brevity,
let \(\mathbf{Z}\) denote \(c^2\norm{g_0}^2\mathbb{I} - g_0 g_0^\top\),
where \(g_0\) and \(g_\omega\) are both vectors \(\in \mathbb{R}^{M \times 1}\), 
and \(\mathbb{I}\) is the \(M \times M\) identity matrix.
Note that the term \(g_0 g_0^\top\) is low rank (in fact rank-1). 
Therefore, \(\mathbf{Z}\) can be interpreted as a rank-1 perturbation of a scaled identity matrix \(c^2 \|g_0\|^2 \mathbb{I}\). 
This is a well-suited setting for the SMW formula, which states
\begin{equation}
    (A + UCV)^{-1} = A^{-1} - A^{-1} U (C^{-1} + V A^{-1} U)^{-1} V A^{-1}.
\end{equation}
Here we have \( A~=~c^2 \norm{g_0}^2 \mathbb{I} \) (diagonal) 
and \( C~=~-1 \), \( U~=~g_0 \), and \( V~=~g_0^\top \).
We first compute \( A^{-1}~=~\frac{1}{c^2 \|g_0\|^2} \mathbb{I}\).
Then substitute \( U~=~V~=~g_0 \) and \( A^{-1} \) to form 
\begin{equation}
    C^{-1}~+~V~A^{-1}~U = -\mathbb{I}~+~\frac{1}{c^2 \|g_0\|^2}~g_0^\top~g_0.
\end{equation}
This yields an scalar, which we invert as
\begin{equation}
    D = \left( -\mathbb{I} + \frac{1}{c^2 \|g_0\|^2} g_0^\top g_0 \right)^{-1}.
\end{equation}
Next, we construct the final inverse using SMW as
\begin{equation}
    \mathbf{Z}^{-1} = \frac{1}{c^2 \|g_0\|^2} \mathbb{I} - \frac{1}{c^4 \|g_0\|^4} g_0 D g_0^\top.
\end{equation}

We can further improve efficiency by computing \( d^{*}~=~\mathbf{Z}^{-1} g_\omega \) without explicitly forming \(\mathbf{Z}^{-1}\).
First compute \(u~=~g_0^\top g_\omega\) and \(v~=~D u\), then we can write
\begin{equation}
    \mathbf{Z}^{-1} g_\omega = \frac{1}{c^2 \|g_0\|^2} g_\omega - \frac{1}{c^4 \|g_0\|^4} g_0 v.
\end{equation}

Hence, using the SMW formula provides us with a closed-form, computationally efficient, and numerically stable~\cite{higham2002accuracy} solution to compute \(d^{*}\) in \cref{eq:dstar}.

\renewcommand{\arraystretch}{0.9}
\begin{algorithm}[tb]
    \caption{\textsc{ConicGrad}}
    \label{alg:algo}
    \begin{algorithmic}[1]
        \STATE {\bfseries Input:} 
            Initial model parameters $\theta_0$, 
            Differentiable task losses $\{L_i\}_{i=1}^{K}$, 
            Learning rates $\eta_1$ and $\eta_2$,
            Decay $\gamma$,
            Cosine similarity constraint $c$.
        \STATE Initialize uniform weights: $\forall i: \omega_i = \frac{1}{K}$
        \FOR{$t=1$ {\bfseries to} $T$}
            \STATE Compute $g_0 = \sum_{i=1}^{K} \nabla_\theta L_i$
            \STATE Compute $g_\omega = \sum_{i=1}^{K} \omega_i \nabla_\theta L_i$
            \STATE Compute $d^* = (c^2\norm{g_0}^2\mathbb{I} - g_0 g_0^\top)^{-1} g_\omega$ 
            \STATE Update the weights $\omega_{t+1}$ using $\eta_2$ via \cref{eq:obj2opt}
            \STATE Recompute $d^{*}$ based on $\omega_{t+1}$
            \STATE Update the model parameters $\theta_{t+1}$ using $\eta_1$
        \ENDFOR
    \end{algorithmic}
\end{algorithm}

\subsection{Convergence Analysis}
We analyze the convergence property of \textsc{ConicGrad}.

\begin{theorem}
Assume individual loss functions \(L_{0}, L_{1}, \cdots L_{K}\) are differentiable on \(\mathbb{R}^{M}\) and their gradients \(\nabla L_{i}(\theta)\) are all \(L\text{-Lipschitz}\) with \(L > 0\), 
i.e., \(\norm{\nabla L_{i}(x) - \nabla L_{i}(y)} \leq L\norm{x - y}\) for \(i~=~0,1,\cdots,K\), where \(L \in (0, \infty)\), and \(L_{0}\) is bounded from below i.e., \(L^{*}_{0}~=~\inf_{\theta \in \mathbb{R}^{m}} L_{0}(\theta) > -\infty\). Then, with a fixed step-size \(\alpha\) satisfying \(0 < \alpha < \frac{1}{L}\), and in case of \(-1 \leq c \leq 1\), \textsc{ConicGrad} satisfies the inequality
\begin{equation}
    \label{eq:rateconverge}
    \sum_{t=0}^{\top}\norm{g_{0}(\theta)}^{2} \leq \frac{2(L_{0}(0)-L_{0}^{*})}{\alpha(2\kappa c-1)(T+1)}.
\end{equation}
\begin{proof}
The proof is provided in~\cref{appx:cagradcongvergence}.
\end{proof}
\end{theorem}

\subsection{Geometric Intuition of \textsc{ConicGrad}}
\label{appxsec:conic_vs_cagrad}
The angular constraint in \textsc{ConicGrad} provides greater flexibility in selecting the gradient update vector \(d\) compared to the directional constraint of CAGrad~\citep{liu2021conflict}. 
To illustrate this, we present a geometric interpretation in a toy \(2\)D vector space, 
demonstrating how \textsc{ConicGrad} permits a broader range of feasible directions while maintaining alignment with the reference objective gradient \(g_0\).
As visualized in \cref{fig:interpretingconicgrad}, \colorbox{conicgrad}{\textsc{ConicGrad}} and \colorbox{cagrad}{CAGrad} define distinct feasible regions. 
For both methods, we set \(c = 0.5\) and assume the main objective gradient vector to be \(g_0 = (1,0)\). 
We observe that CAGrad's constraint defines a Euclidean ball of radius \(0.5\) centered on \(g_0\), limiting feasible updates based on their distance from \(g_0\). 
In contrast, \textsc{ConicGrad} allows any vector within a conic section around \(g_0\), as long as the angle between \(d\) and \(g_0\) does not exceed \(60^\circ\). 
This angular constraint is geometrically intuitive, as it permits a broader range of feasible directions while maintaining alignment with \(g_0\).

\begin{figure}[t]
    \centering
    \includegraphics[width=0.4\textwidth]{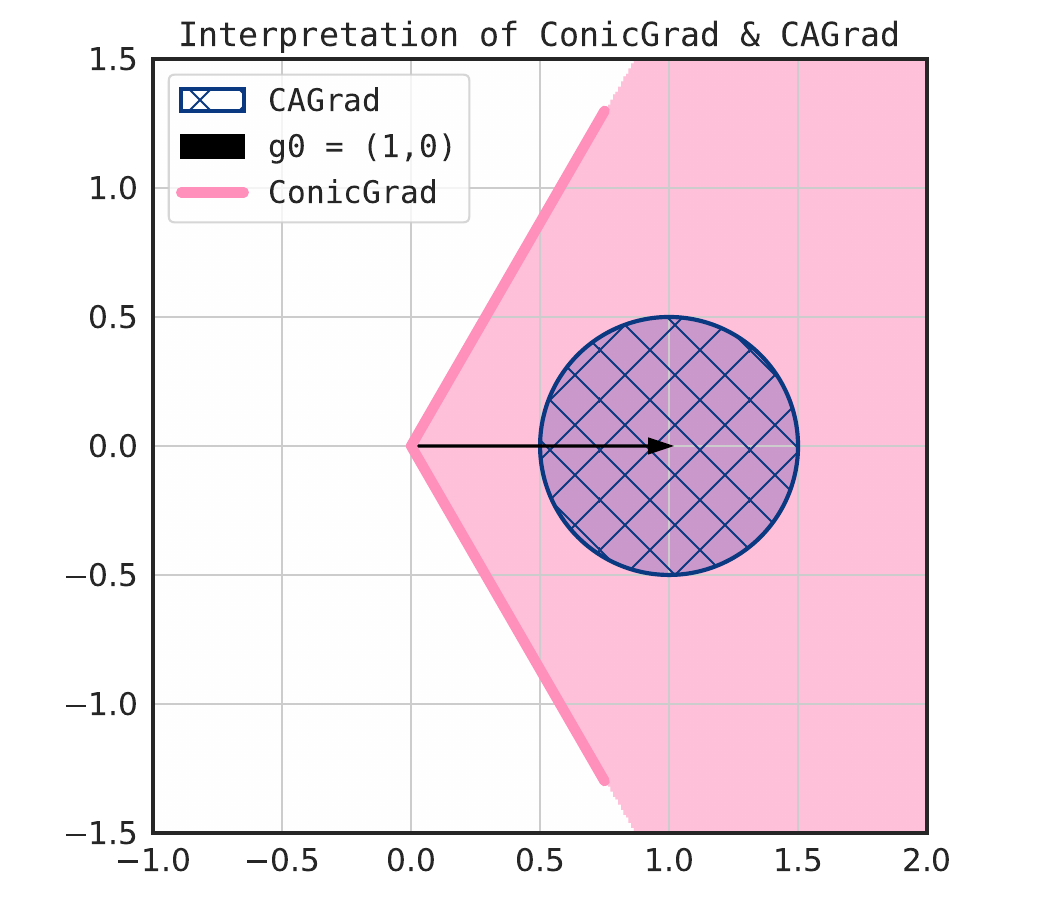} 
    \caption{\textbf{Visualizing Conic vs. Directional Constraints.} 
    We visualize \textsc{ConicGrad} and CAGrad~\citep{liu2021conflict} constraints in a toy setup. The x and y axes denote all possible direction vectors in \(2\)D space \(\mathbb{R}^{2}\), and the plot indicates which vectors in this space satisfy \colorbox{conicgrad}{\textsc{ConicGrad}} and \colorbox{cagrad}{CAGrad} constraints.
    }
    \label{fig:interpretingconicgrad}
\end{figure}

\section{Experimental Results and Discussions}
\label{sec:results}

\begin{table*}[t]
\centering
\caption{\textbf{Results on NYUv2 (3 tasks).} We repeat each experiment over \(3\) different seeds and report the average results. \(\textbf{MR}\) and \textbf{\(\Delta\) m\%} are main MTL metrics.}
\label{tab:nyuv2}
\scalebox{0.8}{
\begin{tabular}{@{}lccccccccccc@{}}
\toprule
 & \multicolumn{2}{c}{Segmentation} & \multicolumn{2}{c}{Depth} & \multicolumn{5}{c}{Surface Normal} &  &  \\ \cmidrule(lr){2-10}
\textbf{Method} & mIoU \(\uparrow\) & Pix Acc \(\uparrow\) & Abs Err \(\downarrow\) & Rel Err \(\downarrow\) & \multicolumn{2}{c}{Angle Dist \(\downarrow\)} & \multicolumn{3}{c}{Within \(t^{\circ}\) \(\uparrow\)} & \textbf{MR} \(\downarrow\) & \textbf{\(\Delta\) m\%} \(\downarrow\) \\
 &  &  &  &  & Mean & Median & 11.25 & 22.5 & 30 &  &  \\ \midrule
STL & \(38.30\) & \(63.76\) & \(0.6754\) & \(0.2780\) & \(25.01\) & \(19.21\) & \(30.14\) & \(57.20\) & \(69.15\) &  &  \\ \midrule
LS & \(39.29\) & \(65.33\) & \(0.5493\) & \(0.2263\) & \(28.15\) & \(23.96\) & \(22.09\) & \(47.50\) & \(61.08\) & \(10.67\) & \(5.59\) \\
SI & \(38.45\) & \(64.27\) & \(0.5354\) & \(0.2201\) & \(27.60\) & \(23.37\) & \(22.53\) & \(48.57\) & \(62.32\) & \(9.89\) & \(4.39\) \\
RLW~\citep{lin2021reasonable} & \(37.17\) & \(63.77\) & \(0.5759\) & \(0.2410\) & \(28.27\) & \(24.18\) & \(22.26\) & \(47.05\) & \(60.62\) & \(13.22\) & \(7.78\) \\
DWA~\citep{liu2019end} & \(39.11\) & \(65.31\) & \(0.5510\) & \(0.2285\) & \(27.61\) & \(23.18\) & \(24.17\) & \(50.18\) & \(62.39\) & \(9.44\) & \(3.57\) \\
UW~\citep{kendall2018multi} & \(36.87\) & \(63.17\) & \(0.5446\) & \(0.2260\) & \(27.04\) & \(22.61\) & \(23.54\) & \(49.05\) & \(63.65\) & \(9.44\) & \(4.05\) \\
MGDA~\citep{sener2018multi} & \(30.47\) & \(59.90\) & \(0.6070\) & \(0.2555\) & \(24.88\) & \(19.45\) & \(29.18\) & \(56.88\) & \(69.36\) & \(7.44\) & \(1.38\) \\
PCGrad~\citep{yu2020gradient} & \(38.06\) & \(64.64\) & \(0.5550\) & \(0.2325\) & \(27.41\) & \(22.80\) & \(23.86\) & \(49.83\) & \(63.14\) & \(10.0\) & \(3.97\) \\
GradDrop~\citep{chen2020just} & \(39.39\) & \(65.12\) & \(0.5455\) & \(0.2279\) & \(27.48\) & \(22.96\) & \(23.38\) & \(49.44\) & \(62.87\) &\(8.89\) & \(3.58\) \\
CAGrad~\citep{liu2021conflict} & \(39.79\) & \(65.49\) & \(0.5486\) & \(0.2250\) & \(26.31\) & \(21.58\) & \(25.61\) & \(52.36\) & \(65.58\) & \(6.33\) & \(0.20\) \\
IMTL-G~\citep{liu2021towards} & \(39.35\) & \(65.60\) & \(0.5426\) & \(0.2256\) & \(26.02\) & \(21.19\) & \(26.20\) & \(53.13\) & \(66.24\) & \(5.56\) & \(-0.76\) \\
NashMTL~\cite{navon2022multi} & \(40.13\) & \(65.93\) & \(0.5261\) & \(0.2171\) & \(25.26\) & \(20.08\) & \(28.40\) & \(55.47\) & \(68.15\) & \(3.67\) & \(-4.04\) \\
FAMO~\citep{NEURIPS2023_b2fe1ee8} & \(38.88\) & \(64.90\) & \(0.5474\) & \(0.2194\) & \(25.06\) & \(19.57\) & \(29.21\) & \(56.61\) & \(68.98\) & \(4.78\) & \(-4.10\) \\ 
SDMGrad~\citep{xiao2024direction} & \(\mathbf{40.47}\) & \(\mathbf{65.90}\) & \(\mathbf{0.5225}\) & \(\mathbf{0.2084}\) & \(25.07\) & \(19.99\) & \(28.54\) & \(55.74\) & \(68.53\) & \(\mathbf{2.78}\) & \(-4.84\) \\
\midrule
\rowcolor{conicgrad} \textbf{\textsc{ConicGrad}} & \(38.67\) & \(65.25\) & \(0.5272\) & \(0.2170\) & \(\mathbf{24.70}\) & \(\mathbf{19.37}\) & \(\mathbf{29.58}\) & \(\mathbf{57.09}\) & \(\mathbf{69.56}\) & \(2.89\) & \(\mathbf{-5.13}\) \\ \bottomrule
\end{tabular}
}

\end{table*}

In this section, we present a comprehensive evaluation of \textsc{ConicGrad} across several standard MTL benchmarks, 
report its performance, compare it to the existing methods, and analyze the results in detail.
We consider two commonly used metrics for evaluating MTL methods~\citep{NEURIPS2023_b2fe1ee8,navon2022multi}: 
(i)~\(\textbf{\(\Delta\) m\%}\) which measures the average per-task performance drop of a method relative to the single task baseline (STL), i.e.,
    \begin{equation*}
        \textbf{\(\Delta\) m\%} = \frac{1}{K}\sum_{k=1}^{K}(-1)^{\delta_{k}}\frac{(M_{m,k}-M_{STL,k})}{M_{STL,k}} \times 100,
    \end{equation*}
    where \(M_{STL,k}\) refers to the value of STL baseline for some metric \(M\) of task \(k\), while \(M_{m,k}\) denotes the value of the method being evaluated for the same metric, 
    and \(\delta_{k}\) is a binary indicator if a metric is better when higher (\(\delta_{k} = 1)\) or lower (\(\delta_{k} = 0)\).
(ii)~Mean Rank (\textbf{MR}) which measures the average rank of each method across different tasks (e.g., \(\textbf{MR} = 1\) when the method ranks first for every task).

\subsection{Toy Example}
\label{sec:toyres}
Given the standard practices in MTL evaluation, we assess \textsc{ConicGrad} on a toy 2-task example~\citep{liu2021conflict}.
This setup consists of two competing objectives that define the overall objective \(\frac{1}{2} \big(\mathcal{L}_1(\theta) + \mathcal{L}_2(\theta)\big)\),
mimicking scenarios where optimization methods must balance conflicting gradients effectively in order to reach the global minimum. 
Failure to do so often results in getting stuck in either of the two suboptimal local minima.
More details on the setup is provided in~\cref{app:toy_detail}.

Using five commonly studied initialization points 
\(\theta_{\text{init}} = \{(-8.5,7.5), (-8.5,-5), (9,9), (-7.5,-0.5), (9,-1)\}\), 
we compare \textsc{ConicGrad} with the following leading methods:
FAMO~\citep{NEURIPS2023_b2fe1ee8}, 
CAGrad~\cite{liu2021conflict}, 
and NashMTL~\citep{navon2022multi}. 
\cref{fig:toyexample} visualizes each method's optimization trajectories, 
illustrating how they handle conflicts in task gradients, as well as their final optimization outcomes.

The results highlight key differences among the methods. 
NashMTL struggles with two initialization points near the two local minima. 
FAMO consistently converges to the Pareto front, but cannot achieve the global minimum. 
This can be explained by its lack of an aligning mechanism with the reference objective gradient \(g_{0}\).
In contrast, both CAGrad and \textsc{ConicGrad} successfully reach the global minimum (\(\scalebox{0.7}{\(\bigstar\)}\) on the Pareto front) for all initialization points. 

Notably, \textsc{ConicGrad} achieves the global minimum significantly faster than CAGrad, 
as evidenced by its respective learning curve on the far-right of \cref{fig:toyexample}.
While FAMO and \textsc{ConicGrad} demonstrate comparable speeds in reaching the Pareto front, 
only \textsc{ConicGrad} consistently converges to the global minimum.
This highlights its effectiveness in balancing objectives and optimizing the overall performance, surpassing competing methods in both convergence speed and outcome.

\subsection{Multi-Task Supervised Learning}

\begin{table*}[!t]
\centering
\caption{\textbf{Results on CityScapes (2 Tasks) and CelebA (40 Tasks).} We repeat each experiment over \(3\) different seeds and report the average results. \(\textbf{MR}\) and \textbf{\(\Delta\) m\%} are main MTL metrics.}
\label{tab:cityscapes}
\scalebox{0.9}{
\begin{tabular}{@{}lccccccccc@{}}
\toprule
 & \multicolumn{6}{c}{\textbf{CityScapes}} && \multicolumn{2}{c}{\textbf{CelebA}} \\ \cmidrule(lr){2-7} \cmidrule(lr){9-10}  
\textbf{Method} & \multicolumn{2}{c}{Segmentation} & \multicolumn{2}{c}{Depth} & \textbf{MR} \(\downarrow\) & \textbf{\(\Delta\) m\%} \(\downarrow\) && \textbf{MR} \(\downarrow\) & \textbf{\(\Delta\) m\%} \(\downarrow\) \\ \cmidrule(lr){2-5}
 & mIoU \(\uparrow\) & Pix Acc \(\uparrow\) & Abs Err \(\downarrow\) & Rel Err \(\downarrow\) &  &  &&  &  \\ \midrule
STL & \(74.01\) & \(93.16\) & \(0.0125\) & \(27.77\) &  &  &&  &  \\
\midrule
LS & \(70.95\) & \(91.73\) & \(0.0161\) & \(33.83\) & \(9.75\) & \(14.11\) && \(6.55\) & \(4.15\) \\
SI & \(70.95\) & \(91.73\) & \(0.0161\) & \(33.83\) & \(9.75\) & \(14.11\) && \(8.0\) & \(7.20\) \\
\begin{tabular}[c]{@{}c@{}}RLW~\citep{lin2021reasonable}\end{tabular} & \(74.57\) & \(93.41\) & \(0.0158\) & \(47.79\) & \(9.0\) & \(24.38\) && \(5.53\) & \(1.46\) \\
\begin{tabular}[c]{@{}c@{}}DWA~\citep{liu2019end}\end{tabular} & \(75.24\) & \(93.52\) & \(0.0160\) & \(44.37\) & \(7.25\) & \(21.45\) && \(7.2\) & \(3.20\) \\
\begin{tabular}[c]{@{}c@{}}UW~\citep{kendall2018multi}\end{tabular} & \(72.02\) & \(92.85\) & \(0.0140\) & \(\mathbf{30.13}\) & \(6.5\) & \(5.89\) && \(6.03\) & \(3.23\) \\
\begin{tabular}[c]{@{}c@{}}MGDA~\citep{sener2018multi}\end{tabular} & \(68.84\) & \(91.54\) & \(0.0309\) & \(33.50\) & \(10.25\) & \(44.14\) && \(11.03\) & \(14.85\) \\
\begin{tabular}[c]{@{}c@{}}PCGrad~\citep{yu2020gradient}\end{tabular} & \(75.13\) & \(93.48\) & \(0.0154\) & \(42.07\) & \(7.25\) & \(18.29\) && \(8.05\) & \(3.17\) \\
\begin{tabular}[c]{@{}c@{}}GradDrop~\citep{chen2020just}\end{tabular} & \(75.27\) & \(93.53\) & \(0.0157\) & \(47.54\) & \(6.5\) & \(23.73\) && \(8.05\) & \(3.29\) \\
\begin{tabular}[c]{@{}c@{}}CAGrad~\citep{liu2021conflict}\end{tabular} & \(75.16\) & \(93.48\) & \(0.0141\) & \(37.60\) & \(6.0\) & \(11.64\) && \(6.42\) & \(2.48\) \\
\begin{tabular}[c]{@{}c@{}}IMTL-G~\citep{liu2021towards}\end{tabular} & \(\mathbf{75.33}\) & \(93.49\) & \(0.0135\) & \(38.41\) & \(4.5\) & \(11.10\) && \(4.92\) & \(0.84\) \\
\begin{tabular}[c]{@{}c@{}}NashMTL~\cite{navon2022multi}\end{tabular} & \(75.41\) & \(\mathbf{93.66}\) & \(\mathbf{0.0129}\) & \(35.02\) & \(\mathbf{2.5}\) & \(6.82\) && \(5.25\) & \(2.84\) \\
\begin{tabular}[c]{@{}c@{}}FAMO~\citep{NEURIPS2023_b2fe1ee8}\end{tabular} & \(74.54\) & \(93.29\) & \(0.0145\) & \(32.59\) & \(6.25\) & \(8.13\) && \(5.03\) & \(1.21\) \\ 
\begin{tabular}[c]{@{}c@{}}SDMGrad~\citep{xiao2024direction}\end{tabular} & \(74.53\) & \(93.52\) & \(0.0137\) & \(34.01\) & \(5.5\) & \(7.79\) && N/A & N/A \\
\midrule
\rowcolor{conicgrad} \textbf{\textsc{ConicGrad}} & \(74.22\) & \(93.05\) & \(0.0133\) & \(30.99\) & \(5.5\) & \(\mathbf{4.53}\) && \(\mathbf{4.0}\) & \(\mathbf{0.10}\) \\ \bottomrule
\end{tabular}
}
\end{table*}

In the supervised MTL setting, we evaluate \textsc{ConicGrad} on three widely used benchmarks, 
namely CityScapes~\citep{cordts2016cityscapes}, CelebA~\citep{liu2015deep}, and NYUv2~\citep{silberman2012indoor}, 
following~\citep{NEURIPS2023_b2fe1ee8,xiao2024direction,navon2022multi,liu2021conflict}. 
\textbf{Cityscapes} includes two tasks: segmentation and depth estimation.
It comprises \(5000\) RGBD images of urban street scenes, each annotated with per-pixel labels. 
\textbf{NYUv2} is another vision-based dataset that involves three tasks: segmentation, depth prediction, and surface normal prediction. 
It contains \(1449\) RGBD images of indoor scenes, with corresponding dense annotations. 
\textbf{CelebA} dramatically increases the number of tasks to \(40\). 
It features approximately \(200\)K images of \(10\)K celebrities, where each face is annotated with \(40\) different binary attributes. 
The task is to classify the presence or absence of these facial attributes for each image. 

Note that CityScapes and NYUv2 are dense prediction tasks, whereas CelebA is a classification task, offering a diverse set of challenges for evaluating MTL methods.
Also note that, while each benchmark has its own set of performance metrics, 
the primary metrics for evaluating MTL methods are \textbf{MR} (Mean Rank) and \(\textbf{\(\Delta\) m\%}\)
(lower is better for both).

We compare \textsc{ConicGrad} against \(12\) multi-task optimization methods 
and a single-task baseline (STL), where a separate model is trained for each task. 
The comparison includes widely recognized MTL approaches such as 
MGDA~\citep{sener2018multi,desideri2012multiple}, 
PCGrad~\citep{yu2020gradient}, 
GradDrop~\citep{chen2020just}, 
CAGrad~\citep{liu2021conflict}, 
IMTL-G~\citep{liu2021towards}, 
NashMTL~\citep{navon2022multi}, 
FAMO~\citep{NEURIPS2023_b2fe1ee8}, and 
SDMGrad~\citep{xiao2024direction}.
Three established methods on gradient manipulation are also evaluated: 
DWA~\citep{liu2019end}, 
RLW~\citep{lin2021reasonable}, 
UW~\citep{kendall2018multi}, 
Additionally, we consider two baseline methods commonly used in MTL literature: 
Linear Scalarization (LS), which minimizes \(L^0\), and 
Scale-Invariant (SI), which minimizes \(\sum_k \log L^k(\theta)\).

\paragraph{Results.} 

On NYUv2 (3 tasks; see~\cref{tab:nyuv2}), \textsc{ConicGrad} delivers strong performance, particularly excelling in surface normal estimation.
While segmentation and depth metrics are slightly behind the best methods, they remain competitive.
Notably, while \textsc{ConicGrad} achieves the second best MR~(\(2.89\)),
it delivers the best average performance improvement across tasks (\(\Delta m\% = -5.13\)), 
demonstrating its ability to outperform the single-task baseline while effectively balancing task trade-offs.

On CityScapes (2 tasks; see~\cref{tab:cityscapes} (left)), \textsc{ConicGrad} achieves a strong balance between segmentation and depth estimation. 
While NashMTL has the best overall MTL performance with an MR of \(2.5\), 
\textsc{ConicGrad} ranks third overall, alongside SDMGrad, 
and just behind IMTL-G. 
Notably, it achieves the best \(\Delta m\%\)~(\(4.53\)), 
reflecting its superior ability to reduce average performance drops across tasks and optimize the trade-offs inherent in MTL. 

\begin{table}[tb]
    \centering
    \caption{\textbf{MTRL results on the Metaworld-10 (MT10) benchmark.} 
    Results are averaged over 10 runs.
    We report the performance of NashMTL~\citep{navon2022multi} as originally presented in their paper, 
    alongside the results reproduced by \citet{NEURIPS2023_b2fe1ee8} in the subsequent rows. 
    }
    \label{tab:mt10}
    \scalebox{0.9}{
    \begin{tabular}{@{}lc@{}}
    \toprule
    \multirow{2}{*}{\textbf{Method}} & \multirow{2}{*}{\begin{tabular}[c]{@{}c@{}}Success \(\uparrow\) \\ (mean \(\pm\) stderr)\end{tabular}} \\
    &  \\ 
    \midrule
    LS (lower bound) & \(0.49 \pm 0.07\) \\
    STL (proxy for upper bound) & \(0.90 \pm 0.03\) \\
    \midrule
    Multiheaded SAC~\citep{yu2020meta} & \(0.61 \pm 0.04\) \\
    PCGrad~\citep{yu2020gradient} & \(0.72 \pm 0.02\) \\
    Soft Modularization~\citep{yang2020multi} & \(0.73 \pm 0.04\) \\
    CAGrad~\citep{liu2021conflict} & \(0.83 \pm 0.05\) \\
    NashMTL~\citep{navon2022multi} (every 1) & \(0.91 \pm 0.03\) \\
    \textcolor{white}{NashMTL}~(reproduced by FAMO) & \(0.80 \pm 0.13\) \\
    NashMTL~\citep{navon2022multi} (every 50) & \(0.85 \pm 0.02\) \\
    \textcolor{white}{NashMTL}~(reproduced by FAMO) & \(0.76 \pm 0.10\) \\
    NashMTL~\citep{navon2022multi} (every 100) & \(0.87 \pm 0.03\) \\
    \textcolor{white}{NashMTL}~(reproduced by FAMO) & \(0.80 \pm 0.12\) \\
    FAMO~\citep{NEURIPS2023_b2fe1ee8} & \(0.83 \pm 0.05\) \\
    SDMGrad~\citep{xiao2024direction} & \(0.84 \pm 0.10\) \\
    \midrule
    \rowcolor{conicgrad} \textbf{\textsc{ConicGrad}} & \(\mathbf{0.89 \pm 0.02}\) \\ \bottomrule
    \end{tabular}
    }
\end{table}

On CelebA (40 tasks; see~\cref{tab:cityscapes} (right)), \textsc{ConicGrad} reports the best rank MR~(\(4.0\)) and lowest \(\Delta m\%\) of \(0.10\). 
This highlights its effectiveness to scale efficiently to large numbers of tasks,
while minimizing performance disparities with STL and achieving nearly optimal task balancing. 

These results highlight \textsc{ConicGrad} as a state-of-the-art method in multi-task learning, 
demonstrating strong performance across diverse MTL benchmarks, from small-scale to large-scale multi-task settings. 
We provide standard errors for our reported performance metrics in \cref{appdix:sl_exps}, 
demonstrating the small variability in our results and highlighting their stability and reliability.


\subsection{Multi-Task Reinforcement Learning}

In addition to supervised multi-task learning benchmarks, we evaluate \textsc{ConicGrad} in a multi-task Reinforcement Learning~(RL)~\citep{sutton2018reinforcement} setting. 
Gradient conflicts are particularly prevalent in RL due to the inherent stochasticity of the paradigm~\citep{yu2020gradient}, 
making it an ideal testbed for optimization strategies that handle such conflicts effectively.
Following prior works~\citep{yu2020gradient,liu2021conflict,NEURIPS2023_b2fe1ee8,xiao2024direction}, 
we benchmark \textsc{ConicGrad} on MetaWorld MT10~\citep{yu2020meta}, 
a widely used MTRL benchmark consisting of 10 robot manipulation tasks, 
each with a distinct reward function.

In accordance with the literature, our base RL algorithm is Soft Actor-Critic (SAC)~\citep{haarnoja2018soft}.
We adopt LS~(i.e., a joint SAC model) as our baseline and STL~(i.e., ten independent SACs, one for each task) as a proxy for skyline.
Other methods we compare to include PCGrad~\citep{yu2020gradient}, CAGrad~\citep{liu2021conflict}, NashMTL~\citep{navon2022multi}, FAMO~\citep{NEURIPS2023_b2fe1ee8}, and SDMGrad~\citep{xiao2024direction}. 
We also compare to an architectural approach to multi-task learning, Soft Modularization~\citep{yang2020multi}, 
wherein a routing mechanism is designed to estimate different routing strategies and all routes are softly combined to form different policies.

\paragraph{Results.}
~\cref{tab:mt10} summarizes the results on the MT10 benchmark. 
While NashMTL reports strong results, \citep{NEURIPS2023_b2fe1ee8} could not reproduce the same performance.
\textsc{ConicGrad} achieves a success rate of \(0.89\) with a standard error of \(0.02\),
outperforming all the contending methods and approaching the STL upper bound (\(0.90 \pm 0.03\)).
Notably, its also exhibits the lowest standard error among all methods which indicates more stable and consistent performance.


\subsection{Scalability Analysis for Larger Models}
\begin{figure}[t]
    \centering
    \includegraphics[width=1.05\linewidth]{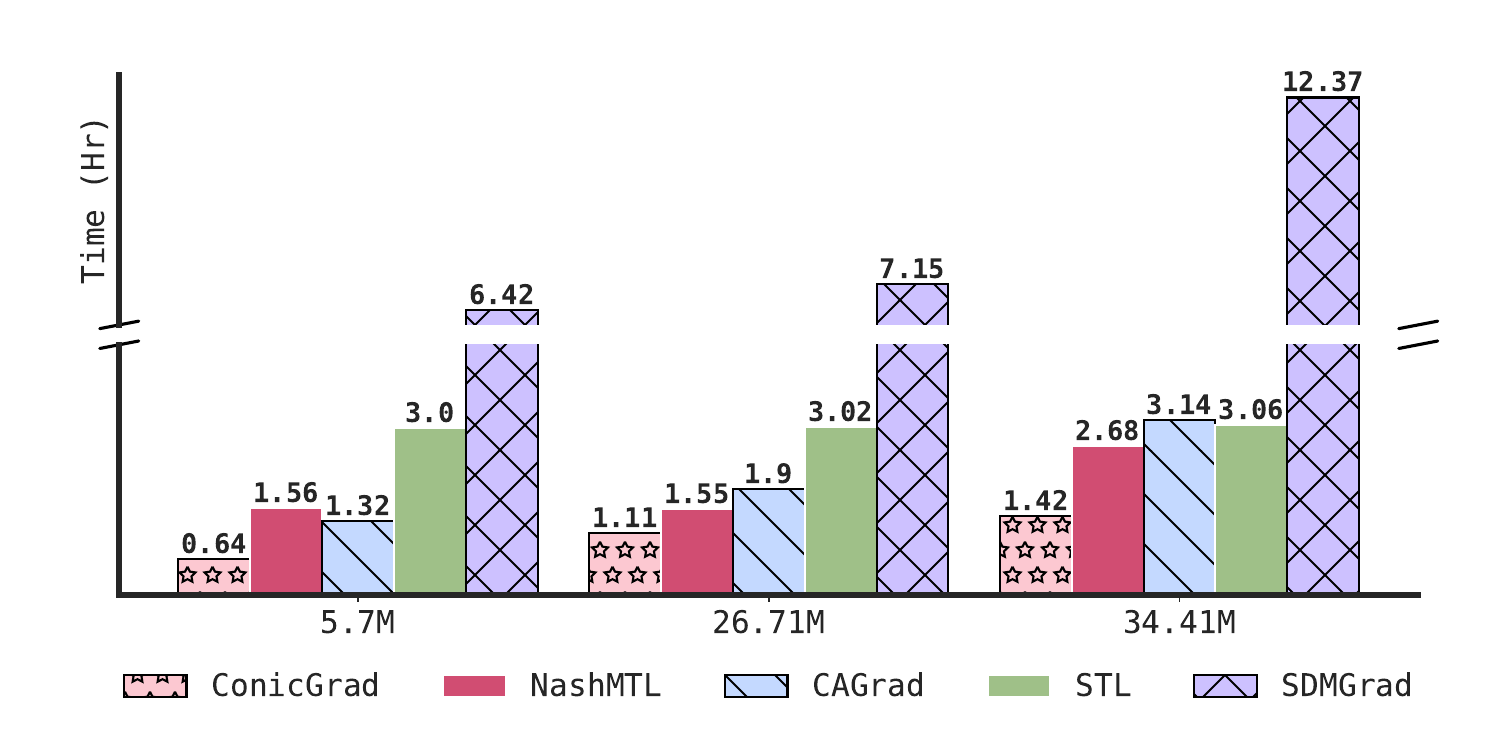}
    \caption{
        \textbf{Scalability Experiments on CelebA.} We measure the computational overhead of MTL methods as the model size increases (in terms of number of parameters) to illustrate how these methods scale.
    }
    \vspace{-10pt}
    \label{fig:scale}
\end{figure}

\begin{figure*}[t!]
    \centering
    \scalebox{0.85}{
    \includegraphics[width=\linewidth]{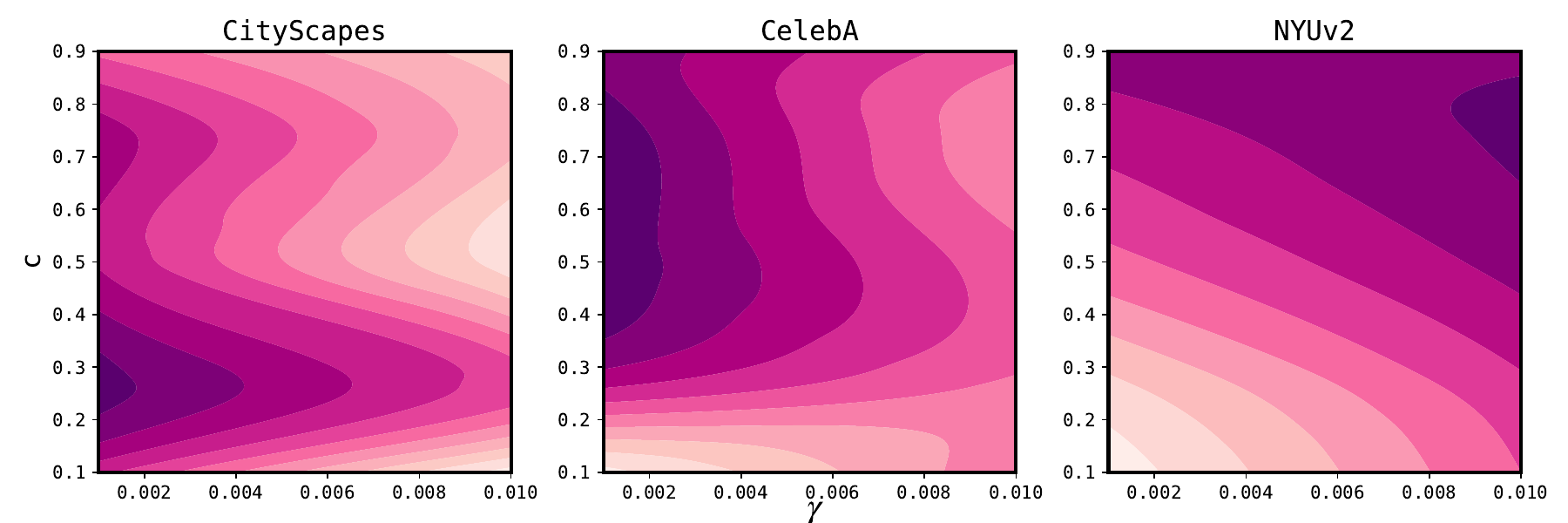}
    }
    \vspace{-0.5em}
    \caption{\textbf{Contour Plots of \(c\) and \(\gamma\) on three MTL benchmarks.} 
    We ablate the hyperparameters \(\gamma \in [0.001, 0.01]\) on the x-axes 
    and \(c \in \{0.1, 0.25, 0.5, 0.75, 0.9\}\) on the y-axes. 
    The raw data consists of discrete values for \(\gamma\) and \(c\) at specific points, 
    and we use interpolation to fill in the gaps 
    to create a continuous surface that reveals how \(\Delta m\%\) (darker areas indicate better performance) varies across the hyperparameter space. 
    }
    \label{fig:ablationplot}
\end{figure*}

We conduct scaling experiments to evaluate the computational overhead incurred by MTL methods as model size increases. 
This analysis is crucial since most of these methods involve direct manipulation of gradients associated with the parameters of the model. 
We compare CAGrad~\citep{liu2021conflict}, NashMTL~\citep{navon2022multi}, SDMGrad~\citep{xiao2024direction}, and \textsc{ConicGrad} along with STL.%
\footnote{
    We excluded FAMO~\citep{NEURIPS2023_b2fe1ee8} from this analysis as it is a zero-order algorithm
    (i.e., does not require explicit computation of gradients of \(\theta\)),
    while the rest of the methods are first-order.
}
CelebA~\citep{liu2015deep} is chosen as the evaluation benchmark due to its high task count (i.e., 40). 
The base model has \(5.2\)M parameters (measured using the ptflops package~\citep{ptflops}), 
and we create two scaled variants with \(26.71\)M and \(34.41\)M parameters, 
representing roughly \(5\times\) and \(7\times\) the base model size, 
by increasing the number of layers and neurons. 
For each method and model size, we measure the average time per epoch over two runs. 

\cref{fig:scale} illustrates the per-epoch time (in hours) for each algorithm as the underlying model size increases. 
Our method, \textsc{ConicGrad}, remains computationally efficient as the number of parameters grows, 
while other methods experience slowdowns.
In particular, SDMGrad~\citep{xiao2024direction} requires significantly more time due to its reliance on multiple forward passes for gradient estimation.

\subsection{Ablation Study}

We examine the effect of two key hyperparameters of \textsc{ConicGrad}, 
namely \(c\) which controls the maximum permissible angle between
the gradient update vector \(d\) and 
the reference gradient vector \(g_{0}\), 
and the regularization coefficient \(\gamma\). 
We explore \(\gamma \in [0.001, 0.01]\) and \(c \in \{0.1, 0.25, 0.5, 0.75, 0.9\}\) as the space of admissible values. 
In~\cref{fig:ablationplot}, we observe that smaller values of \(\gamma\) are preferred for the CityScapes and CelebA benchmarks, 
while a larger \(\gamma\) enhances performance on NYUv2. 
As for the conic constraint, \(c \geq 0.5\) is generally preferred for CelebA and NYUv2 benchmarks, 
enforcing the angle between \(d\) and \(g_0\) to remain below \(60^\circ\). 
In contrast, for CityScapes, a smaller \(c\) (i.e., \(c=0.25\)) leads to better performance, 
suggesting that the update vector \(d\) should be allowed to deviate more from \(g_0\) in this benchmark.

\section{Conclusion}
\label{sec:conclusion}


In this work, we explored Multi-Task Learning (MTL) 
through the lens of Multi-Objective Optimization (MOO) and introduced \textsc{ConicGrad}.
A fundamental challenge in MTL is gradient conflicts, where task gradients may point in opposing directions, 
making it difficult to find a unique gradient update vector \(d\) that improves all tasks simultaneously.
To address this, \textsc{ConicGrad} analyzes the evolving relationships between task-specific gradients as optimization progresses, 
and dynamically computes \(d\) at each training step.
\textsc{ConicGrad} offers a geometrically interpretable solution by enforcing an angular constraint, 
ensuring that \(d\) remains within a cone defined by an angle of at most \(\arccos(c)\) relative to the reference objective gradient \(g_0\). 
This formulation preserves alignment with \(g_0\) while still permitting adaptive adjustments to task-specific contributions. 
Additionally, we demonstrate that not only \textsc{ConicGrad} is computationally efficient, but also scales effectively to high-dimensional parameter spaces. 
Evaluations on standard supervised and reinforcement learning benchmarks demonstrate that \textsc{ConicGrad} consistently outperforms state-of-the-art methods in most cases, 
while remaining competitive in others.
 
\paragraph{Limitations and Future Work.}
Our method relies on the cone angle parameter \(c\), which influences the alignment constraint. 
While \textsc{ConicGrad} has demonstrated strong performance across various tasks with a small set of tried values for \(c\), 
a promising avenue for future work is the development of strategies to dynamically adjust \(c\) during training. 
This dynamic adaptation can enhance the algorithm's ability to navigate the loss landscape more effectively and potentially accelerate convergence.



\section*{Impact Statement}

This paper presents work whose goal is to advance the field of Machine Learning. 
There are many potential societal consequences of our work, none which we feel must be specifically highlighted here.

\bibliography{example_paper}
\bibliographystyle{icml2025}

\newpage
\appendix
\onecolumn
\section*{Appendix}

\section{Proofs}

\subsection{Optimal Gradient Direction}
\label{appxderiv:proofsconicgrad}
\begin{proposition}
Given the optimization problem in~\cref{eq:org_eq}, its Lagrangian in~\cref{eq:primal}, and assuming the Slater condition holds, the dual of the primal problem in~\cref{eq:obj}, then the optimal update direction \(d^{*}\) is given by
\begin{equation*}
    d^{*} = \frac{1}{\lambda}\left(c^{2}\norm{g_{0}}^{2}\mathbb{I} - g_{0}g_{0}^\top\right)^{-1}g_{\omega}
\end{equation*}
\end{proposition}

\begin{proof}
We re-produce the Lagrangian from~\cref{eq:primal} as
\begin{equation}
\label{eq:lagrangian_expanded}
    \max\limits_{d \in \mathbb{R}^{M}} \min\limits_{i \in [K]} \langle g_{i},d\rangle - \frac{\lambda}{2} (c^2 \norm{g_{0}}^2 \norm{d}^2 - \norm{g_0^\top d}^2).
\end{equation}

We take the partial derivative of~\cref{eq:lagrangian_expanded} with respect to \(d\) while keeping \(\lambda\) and \(\omega\) fixed, i.e., \(\frac{\partial}{\partial(d)}\) and get
\begin{align}
\label{eq:der_wrt_d}
    &g_{\omega} - \lambda\left[c^{2}\norm{g_{0}}^{2} d - g_{0}g_{0}^\top d\right] = 0,\\
    &g_{\omega} - \lambda c^{2}\norm{g_{0}}^{2} d + \lambda g_{0}g_{0}^\top d = 0.
\end{align}

Since we want to find \(d^{*}\), we collect all the terms dependent on \(d\) and re-arrange as
\begin{align}
    \label{eq:sim_der_wrt_d4}
    &g_{\omega} = \lambda c^{2}\norm{g_{0}}^{2} d - \lambda g_{0}g_{0}^\top d,\\
    &g_{\omega} = \lambda d (c^{2}\norm{g_{0}}^{2}\mathbb{I} - g_{0}g_{0}^\top).
\end{align}

Now, we can write \(d^{*}\) as
\begin{align}
    d^{*} &= \frac{g_{\omega}}{\lambda (c^{2}\norm{g_{0}}^{2}\mathbb{I} - g_{0}g_{0}^\top)}, \\
    d^{*} &= \frac{1}{\lambda}\left(c^{2}\norm{g_{0}}^{2}\mathbb{I} - g_{0}g_{0}^\top\right)^{-1}g_{\omega},
\end{align}
and we arrive at the equation we set out to prove.
\end{proof}

\subsection{Optimizing for \(\lambda\)}
\label{app:lambda_opt}

To optimize \(\lambda\), we first substitute \(d^*\) from \cref{eq:dstar} into \cref{eq:obj_without_max}.  
To simplify the resulting expression, 
we define \(\mathbf{Z} \coloneqq \left(c^{2}\norm{g_{0}}^{2}\mathbb{I} - g_{0}g_{0}^\top\right)\), 
which is independent of \(\lambda\), following a similar approach as in \cref{sec:smw}, then

\[
E(\lambda)\;=\;\frac{1}{\lambda}\,\bigl(g_{w}^{\top} \mathbf{Z}^{-1}\,g_{w}\bigr)\;-\;\frac{\lambda}{2}\,\Bigl(\,c^{2}\,\|g_{0}\|^{2}\,\Bigl\|\tfrac{\mathbf{Z}^{-1}\,g_{w}}{\lambda}\Bigr\|^{2}\;-\;\Bigl\|\tfrac{g_{0}^{\top}\mathbf{Z}^{-1}\,g_{w}}{\lambda}\Bigr\|^{2}\Bigr).
\]

This further simplifies to

\[
E(\lambda) 
\;=\;
\underbrace{\frac{g_{w}^{\top}\mathbf{Z}^{-1}\,g_{w}}{\lambda}}_{\text{Term 1}}
\;-\;
\underbrace{\frac{1}{2\,\lambda}
\Bigl[
   c^{2}\,\|g_{0}\|^{2}\,\|\mathbf{Z}^{-1}\,g_{w}\|^{2} 
   - 
   \bigl(g_{0}^{\top}\mathbf{Z}^{-1}\,g_{w}\bigr)^{2}
\Bigr]}_{\text{Term 2}}
\;=\;
\frac{1}{\lambda}\,\Bigl[
  g_{w}^{\top}\mathbf{Z}^{-1}\,g_{w}
  \;-\;
  \tfrac{1}{2}\,
  \bigl(
    c^{2}\,\|g_{0}\|^{2}\,\|\mathbf{Z}^{-1}\,g_{w}\|^{2}
    -
    (g_{0}^{\top}\mathbf{Z}^{-1}\,g_{w})^{2}
  \bigr)
\Bigr].
\]

\paragraph{Derivative w.r.t. \(\lambda\)}

Let the constant part (independent of \(\lambda\)) be denoted by \(C\), then we can write

\[
C 
\;=\;
g_{w}^{\top}\mathbf{Z}^{-1}\,g_{w}
\;-\;
\tfrac{1}{2}\,
\Bigl[
  c^{2}\,\|g_{0}\|^{2}\,\|\mathbf{Z}^{-1}\,g_{w}\|^{2}
  - 
  (g_{0}^{\top}\mathbf{Z}^{-1}\,g_{w})^{2}
\Bigr].
\]

So, \(E(\lambda) \;=\; \frac{C}{\lambda}\), and its derivative is
\[
\frac{d}{d\lambda}\,E(\lambda)
\;=\;
\frac{d}{d\lambda}\Bigl(\tfrac{C}{\lambda}\Bigr)
\;=\;
C \,\cdot \frac{d}{d\lambda}\Bigl(\tfrac{1}{\lambda}\Bigr)
\;=\;
-\frac{C}{\lambda^{2}}.
\]
However, since \(C \neq 0\), no finite \(\lambda\) satisfies the stationary condition. This exemplifies why, in Machine Learning (ML), iterative optimization is often preferred over algebraic closed-form solutions,
as many Lagrange-like formulations do not yield a single closed-form multiplier. Thus, in this work, we set \(\lambda = 1\) and found it to perform well empirically.

\subsection{Convergence Analysis of \textsc{ConicGrad}}
\label{appx:cagradcongvergence}
We borrow from~\citep{liu2021conflict} the style and language for the purpose of analyzing convergence rate of \textsc{ConicGrad}. We abuse the notation and assume that \(L_{0}\) denotes a general function which has associated gradient \(g_{0} = \nabla L_{0}\).

\begin{theorem}
Assume individual loss functions \(L_{0}, L_{1}, \cdots L_{K}\) are differentiable on \(\mathbb{R}^{M}\) and their gradients \(\nabla L_{i}(\theta)\) are all \(L-Lipschitz\), i.e., \(\norm{\nabla L_{i}(x) - \nabla L_{i}(y)} \leq L\norm{x - y}\) for \(i = 0,1,\cdots,K\), where \(L \in (0, \infty)\), and \(L_{0}\) is bounded from below i.e., \(L^{*}_{0} = \inf_{\theta \in \mathbb{R}^{m}} L_{0}(\theta) > -\infty\). Then, with a fixed step-size \(\alpha\) satisfying \(0 < \alpha < \frac{1}{L}\), and in case of \(-1 \leq c \leq 1\), \textsc{ConicGrad} satisfies
\begin{equation}
    \label{eq:rateconverge}
    \sum_{t=0}^{T}\norm{g_{0}(\theta)}^{2} \leq \frac{2(L_{0}(0)-L_{0}^{*})}{\alpha(2\kappa c-1)(T+1)}.
\end{equation}

\begin{proof}
Consider the optimization step to be \(t^{th}\) and let \(d^{*}(\theta_{t})\) be the update direction obtained by solving~\cref{eq:org_eq}, then we can write
\begin{align*}
    L_{0}(\theta_{t+1}) - L_{0}(\theta_{t}) &\leq L_{0}(\theta_{t} - \alpha d^{*}(\theta_{t})) - L_{0}(\theta_{t}) \\
    L_{0}(\theta_{t+1}) - L_{0}(\theta_{t}) &\leq L_{0}(\theta_{t} - \alpha d^{*}(\theta_{t})) - L_{0}(\theta_{t}) \leq -\alpha g_{0}(\theta_{t})^\top d^{*}(\theta_{t}) + \frac{L}{2}\norm{-\alpha d^{*}(\theta_{t})}^{2} \quad \text{\textit{by smoothness}} \\
    & \leq -\alpha g_{0}(\theta_{t})^\top d^{*}(\theta_{t}) + \frac{L\alpha^{2}}{2}\norm{d^{*}(\theta_{t})}^{2} \quad \text{\textit{re-arranging}}\\
    & \leq -\alpha g_{0}(\theta_{t})^\top d^{*}(\theta_{t}) + \frac{\alpha}{2}\norm{d^{*}(\theta_{t})}^{2} \quad \text{\textit{by}}~\alpha \leq \frac{1}{L} \\
    & \leq -\alpha c\norm{g_{0}(\theta_{t})}\norm{d^{*}(\theta_{t})} + \frac{\alpha}{2}\norm{d^{*}(\theta_{t})}^{2} \quad \text{\textit{by} constraint in~\cref{eq:org_eq}} \\
    & = -\alpha c\norm{g_{0}(\theta_{t})}\kappa\norm{g_{0}(\theta_{t})} + \frac{\alpha}{2}(\kappa\norm{g_{0}(\theta_{t})})^{2} \quad \text{\textit{because we enforce}}~\norm{d^{*}(\theta_{t})} \approx \kappa\norm{g_{0}(\theta)} \\
    & = -\alpha\kappa c\norm{g_{0}(\theta_{t})}^{2} + \frac{\alpha}{2}\kappa^{2}\norm{g_{0}(\theta_{t})}^{2} \\
    & = - (\alpha\kappa c - \frac{\alpha}{2}) \norm{g_{0}(\theta_{t})}^{2} \\
    & = - \left(\frac{2\alpha\kappa c - \alpha}{2}\right) \norm{g_{0}(\theta_{t})}^{2},
\end{align*}
where \(\kappa\) is some constant that \(\norm{d^{*}(\theta_{t})}\) approximates \(\norm{g_{0}(\theta)}\) with. Note that this follows from the normalization term \(\Tilde{d} = d^{*}\frac{\norm{g_{0}}}{\norm{d^{*}}}\). Using telescopic sums, we have \(L_{0}(\theta_{T+1}) - L_{0}(0) = - \left(\frac{2\alpha\kappa c - \alpha}{2}\right) \sum_{t=0}^{T}\norm{g_{0}(\theta_{t})}^{2}\). Therefore,
\begin{align*}
    &\min_{t\leq T}\norm{g_{0}(\theta_{t})}^{2} \leq \frac{1}{T+1}\sum_{t=0}^{T}\norm{g_{0}(\theta_{t})}^{2} \leq \frac{2(L_{0}(0)-L_{0}(\theta_{T+1}))}{2\alpha\kappa c - \alpha(T+1)}\\
    & = \min_{t\leq T}\norm{g_{0}(\theta_{t})}^{2} \leq \frac{1}{T+1}\sum_{t=0}^{T}\norm{g_{0}(\theta_{t})}^{2} \leq \frac{2(L_{0}(0)-L_{0}(\theta_{T+1}))}{\alpha(2\kappa c-1)(T+1)}\\
\end{align*}
Therefore, if \(L_{0}\) is bounded from below, then \(\min_{t\leq T}\norm{g_{0}(\theta_{t})}^{2} = \mathcal{O}(\frac{1}{T})\).
\end{proof}
\end{theorem}

\begin{proposition}\label{appxprop:cgandgd}
When \(c = 1\) in \textsc{ConicGrad}'s optimization objective (\cref{eq:org_eq}), then the update direction vector and average gradient are collinear and positively aligned. Therefore, \textsc{ConicGrad} recovers gradient descent objective with \(d = g_{0}\).
\begin{proof}
From~\cref{eq:org_eq}, recall that
\begin{equation*}
    \max\limits_{d \in \mathbb{R}^K} \min\limits_{i \in [T]} \langle g_i,d\rangle \quad s.t. \quad \frac{\langle g_0,d\rangle}{\norm{g_0} \norm{d}} \geq c,
\end{equation*}
where \(c\) refers to a value between \([-1,1]\) because \(-1\leq cos(\theta) \leq 1\) is the range of the cosine function. In gradient descent for MTL, the update direction is the convex combination of different objective functions. Assume a simple 2 task problem with gradients \(g_{1}\) for task 1 and \(g_{2}\) for task 2, then the update vector \(d\) with gradient descent is given by
\begin{equation}
\label{eq:gdmtl}
    d = \frac{g_{1}+g_{2}}{2}.
\end{equation}

We substitute \(d = \frac{g_{1}+g_{2}}{2}\) and \(g_{0} = \frac{g_{1}+g_{2}}{2}\) in~\cref{eq:org_eq}, and simplify the objective by first computing the numerator \(\langle g_{0}, d\rangle\) as,
\begin{align}
    \langle g_{0}, d\rangle &= \langle \frac{g_{1}+g_{2}}{2}, \frac{g_{1}+g_{2}}{2}\rangle \\
    \label{eq:innerprod}
    &= \frac{1}{4}\left(\left\Vert g_{1}\right\Vert^{2} + 2\langle g_{1},g_{2}\rangle + \left\Vert g_{2}\right\Vert^{2}\right).
\end{align}
Then, we compute both values in the denominator, i.e., \(\left\Vert g_{0}\right\Vert\) and \(\left\Vert d\right\Vert\) as,
\begin{align}
    \left\Vert g_{0}\right\Vert &= \left\Vert \frac{g_{1}+g_{2}}{2}\right\Vert = \sqrt{\langle\frac{g_{1}+g_{2}}{2},\frac{g_{1}+g_{2}}{2}\rangle} \\
    &= \sqrt{\frac{1}{4}\left(\left\Vert g_{1}\right\Vert^{2} + 2\langle g_{1},g_{2}\rangle + \left\Vert g_{2}\right\Vert^{2}\right)} \quad \because~\text{\cref{eq:innerprod}}
\end{align}
and since \(d = g_{0}\), we can say \(\left\Vert d\right\Vert = \left\Vert g_{0}\right\Vert\). We then substitute all the expressions in \textsc{ConicGrad's} objective and get
\begin{align}
    &= \frac{\frac{1}{4}\left(\left\Vert g_{1}\right\Vert^{2} + 2\langle g_{1},g_{2}\rangle + \left\Vert g_{2}\right\Vert^{2}\right)}{\sqrt{\frac{1}{4}\left(\left\Vert g_{1}\right\Vert^{2} + 2\langle g_{1},g_{2}\rangle + \left\Vert g_{2}\right\Vert^{2}\right)} \cdot \sqrt{\frac{1}{4}\left(\left\Vert g_{1}\right\Vert^{2} + 2\langle g_{1},g_{2}\rangle + \left\Vert g_{2}\right\Vert^{2}\right)}} \\
    \label{eq:conicgradgdfinal}
    &= \frac{\frac{1}{4}\left(\left\Vert g_{1}\right\Vert^{2} + 2\langle g_{1},g_{2}\rangle + \left\Vert g_{2}\right\Vert^{2}\right)}{\frac{1}{4}\left(\left\Vert g_{1}\right\Vert^{2} + 2\langle g_{1},g_{2}\rangle + \left\Vert g_{2}\right\Vert^{2}\right)} = 1.
\end{align}
\cref{eq:conicgradgdfinal} implies that \(\frac{\langle g_0,d\rangle}{\norm{g_0} \norm{d}} = 1\) which is only true when the angle between \(d\) and \(g_{0}\) is \(0^{\circ}\), because \(\cos(0^{\circ}) = 1\), and both \(d\) and \(g_{0}\) are exactly collinear and positively aligned. Therefore, \textsc{ConicGrad} recovers the the gradient descent objective in MTL setting in this case.
\end{proof}
\end{proposition}

\begin{proposition}\label{appxprop:slatercond} If \(c < 1\), then the direction vector \(d_{s} = \alpha g_{0}\) (with \(\alpha > 0\)) satisfies
\begin{align*}
    \frac{\langle g_{0},d_{s}\rangle}{\left\Vert g_{0}\right\Vert \left\Vert d_{s}\right\Vert} > c,
\end{align*}
i.e., \( \langle g_{0},d_{s}\rangle - c\left\Vert g_{0}\right\Vert \left\Vert d_{s}\right\Vert > 0\), and Slater condition holds.
\begin{proof}
Let \(g_{0},g_{1},\cdots,g_{k} \in \mathbb{R}^{m}\), and at each iteration we solve
\begin{align*}
    \max\limits_{d \in \mathbb{R}^K} \min\limits_{i \in [T]} \langle g_i,d\rangle \quad s.t. \quad \frac{\langle g_0,d\rangle}{\norm{g_0} \norm{d}} > c \quad \text{from~\cref{eq:org_eq}}.
\end{align*}
where we replace \(\geq\) with \(>\) because strict inequality is what needs to hold. From this we can say
\begin{align}
    \langle g_0,d\rangle > c\norm{g_0}\norm{d},\\
    \label{eq:slatercondition}
    = \langle g_0,d\rangle - c\norm{g_0}\norm{d} > 0.
\end{align}
We aim to show the existence of at least one point \(d_{s}\) in the feasible set satisfying~\cref{eq:slatercondition}, 
i.e., \(\langle g_0,d_{s}\rangle - c\norm{g_0}\norm{d_{s}} > 0\). Let \(d_{s} = \alpha g_{0} \in \mathbb{R}^m\) be the vector \(g_{0}\) scaled by some scalar \(\alpha > 0\). We can then re-write~\cref{eq:org_eq} as
\begin{align}
    \label{eq:maincondition}
    \frac{\langle g_{0},d_{s}\rangle}{\left\Vert g_{0}\right\Vert \left\Vert d_{s}\right\Vert} > c.
\end{align}
We substitute the value of \(d_{s}\) and simplify both numerator and denominator as \(\langle g_{0},d_{s}\rangle = \langle g_{0},\alpha g_{0}\rangle = \alpha\norm{g_{0}}^{2}\), and \(\norm{d_{s}} = \norm{\alpha g_{0}} = \alpha\norm{g_{0}}\). We now re-write~\cref{eq:slatercondition} as
\begin{align}
    & \left(\alpha\norm{g_{0}}^{2}\right) - c\norm{g_0}\left(\alpha\norm{g_{0}}\right) > 0\\
    &= \alpha\norm{g_{0}}^{2}  - c\alpha\norm{g_0}^2 > 0 \\
    \label{eq:slatercondition2}
    &= \alpha\norm{g_0}^2(1-c) > 0.
\end{align}
In~\cref{eq:slatercondition2}, as \(\alpha\norm{g_0}^2 > 0\), therefore \(1-c > 0\) or \(c < 1\). Hence, we see that \(d_{s}\) strictly satisfies the inequality \(\langle g_{0},d\rangle - c\left\Vert g_{0}\right\Vert \left\Vert d\right\Vert > 0\) and~\cref{eq:maincondition}, and \(d_{s}\) is a feasible point. Since Slater condition holds, we have strong duality.
\end{proof}
\begin{remark}
In~\cref{appxprop:cgandgd}, we show that in the case of \(c = 1\), \textsc{ConicGrad} recovers gradient descent with \(d = g_{0}\). However, in the case of \(c > 1\), we see that it implies \(\frac{\langle g_0,d\rangle}{\norm{g_0} \norm{d}} \geq c > 1\), and this is impossible by Cauchy–Schwarz inequality (which states \(\langle g_0,d\rangle \leq \norm{g_{0}}\norm{d}\)). Therefore, the feasible set is empty and no solution exists.
\end{remark}
\end{proposition}

\section{ Detailed Look at \textsc{ConicGrad}}
\subsection{Exploring \textsc{ConicGrad}'s Hyperparameter \(c\)}
The range of \(c\) is \(\in [-1,1]\), and hence no solution exists outside of this range. 
In practice, we restrict \(c \in (0, 1]\) to avoid negative correlation. 
In~\cref{fig:geometricinterpretation} we visualize the range of admissible update directions (the \colorbox{conicgrad}{pink} area), given a certain \(c\). 

\begin{itemize}
\setlength\itemsep{0em}
    \item When \(c = 1\), then both the update direction vector \(d\) and average gradient \(g_{0}\) are collinear, i.e., they lie on top of each other and the angle is zero. 
    \item In case of \(c = 0.5\), the admissible vectors form acute angles, while in the case of \(c < 0\) (e.g., \(c = -0.3\)), the update direction vectors form obtuse angles with the average gradient. 
    \item When \(c = 0\), all the admissible vectors are orthogonal to \(g_{0}\). 
    \item As the value of \(c\) decreases below \(0\) but \(\geq -1\), \(d\) vectors that are negatively correlated become admissible.
    \item In the extreme case of \(c = -1\), the entire region becomes feasible.
\end{itemize}

\label{appxsec:geometryconic}
\begin{figure*}[!b]
    \centering
    \includegraphics[width=\linewidth]{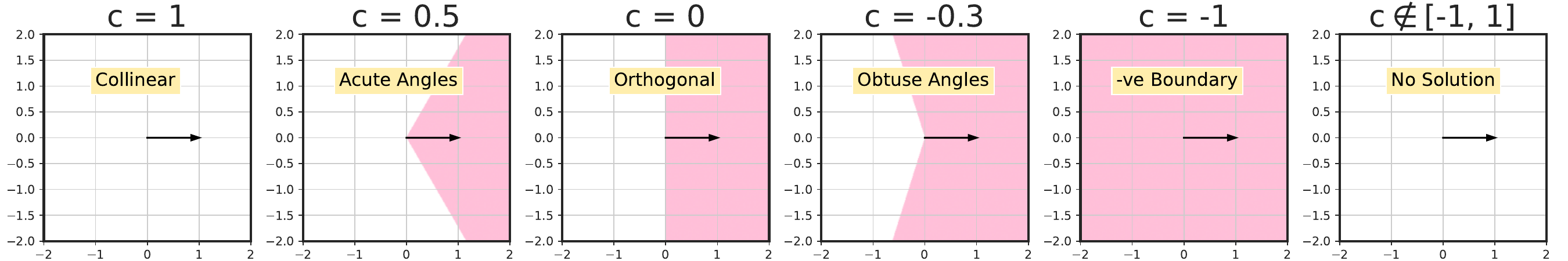}
    \caption{\textbf{Geometric Interpretation of \textsc{ConicGrad}.} 
        The black arrow indicates the main objective gradients vector \(g_0\), and the region covered in \colorbox{conicgrad}{pink} indicates the valid region. We plot several values of \(c\) to visualize how the region of admissible update direction vectors changes.
    }
    \label{fig:geometricinterpretation}
\end{figure*}

\begin{figure}[!b]
\vspace{-15pt}
    \centering
    \includegraphics[width=\linewidth]{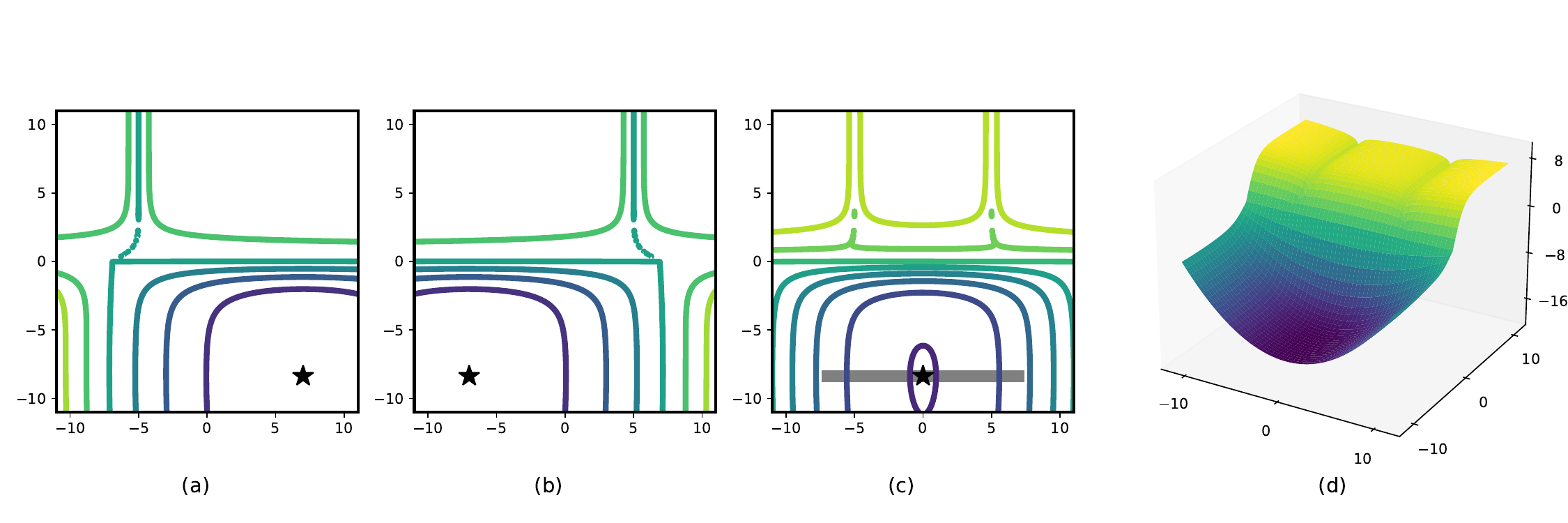}
    \caption{\textbf{Loss Landscape of the Toy Example.} 
    (a) First objective, \( \mathcal{L}_1(\theta) \);
    (b) Second objective, \( \mathcal{L}_2(\theta) \);
    (c) The overall objective, \( \frac{1}{2} \big(\mathcal{L}_1(\theta) + \mathcal{L}_2(\theta)\big) \); and 
    (d) 3D plot of (c).
    The global minima for (a), (b), and (c) are represented by \(\scalebox{0.7}{\(\bigstar\)}\),
    and the Pareto front for (c) is denoted by a gray line.
    Balancing the conflicting gradients while navigating the overall objective is crucial; 
    since without it, the optimization process is prone to getting stuck in either of the two suboptimal local minima.
    }
    \vspace{-15pt}
    \label{fig:toy_setting}
\end{figure}

\section{Toy Example for Multi-Task Optimization}
\label{app:toy_detail}

We adopt a simple yet insightful toy example from \citep{liu2021conflict} 
which serves as a minimal, interpretable framework for evaluating and visualizing the performance of multi-task learning methods under controlled conditions. 
Figure \ref{fig:toy_setting} illustrates the loss landscapes of each individual task \(\mathcal{L}_1(\theta)\) and \(\mathcal{L}_2(\theta)\) (in panels (a) and (b) respectively), 
as well as the combined objective (in panel (c)). 
The global minima for each task and the combined objective are indicated by \scalebox{0.7}{\(\bigstar\)}, 
with the gray line in (c) representing the Pareto front.
Panel (d) in \cref{fig:toy_setting} represents the 3D plot of the combined task.

The two competing objectives in the overall objective \(\frac{1}{2} \big(\mathcal{L}_1(\theta) + \mathcal{L}_2(\theta)\big)\) 
simulate scenarios where optimization methods must balance conflicting gradients effectively in order to reach the global minimum. 
Failing to do so often results in getting stuck in either of the two suboptimal local minima.
In \cref{fig:toyexample}, we visualize optimization trajectories for \textsc{ConicGrad}, 
along with those for FAMO, CAGrad, and NashMTL using five standard initialization points. 
The results reveal that NashMTL struggles with two initialization points for which it gets stuck in local minima.
FAMO reaches the Pareto front but fails to find the global optimum.
In contrast, both CAGrad and \textsc{ConicGrad} reach the global minima for all initialization points,
with \textsc{ConicGrad} convergings significantly faster than CAGrad.

Next, we provide the exact mathematical formulation of the two tasks.

\paragraph{Problem Formulation}
The toy example consists of a two-dimensional input vector 
\( \theta = [\theta_1, \theta_2] \), 
and two objectives 
\begin{align*}
    \mathcal{L}_1(\theta) &= ( f_1 \cdot c_1 + g_1 \cdot c_2 ) 
    \quad \textrm{and} \\
    \mathcal{L}_2(\theta) &= ( f_2 \cdot c_1 + g_2 \cdot c_2 ),
\end{align*}
where
\begin{align*}
    f_1 &= \log\Big(\text{max}(| 0.5(-\theta_1 - 7) - \tanh(-\theta_2) |, \epsilon)\Big) + 6, \\
    f_2 &= \log\Big(\text{max}(| 0.5(-\theta_1 + 3) + \tanh(-\theta_2) + 2 |, \epsilon)\Big) + 6, 
\end{align*}
\begin{align*}
    g_1 &= \frac{(-\theta_1 + 7)^2 + 0.1(-\theta_2 - 8)^2}{10} - 20, \\
    g_2 &= \frac{(-\theta_1 - 7)^2 + 0.1(-\theta_2 - 8)^2}{10} - 20, 
\end{align*}
and
\begin{align*}
    c_1 &= \text{max}\Big(\tanh(0.5 \theta_2), 0\Big), \\
    c_2 &= \text{max}\Big(\tanh(-0.5 \theta_2), 0\Big)
\end{align*}
with \( c_1 \) and \( c_2 \) as binary switching terms, 
and \(\epsilon = 5 \times 10^{-6}\) ensures numerical stability. 

\section{More Details on the Experiments}
\label{appdix:practicaldetails}

In this section, we discuss the practical details of our proposed method, \textsc{ConicGrad}. 
We implement our method based on the library that the authors of NashMTL~\citep{navon2022multi} released and follow the settings therein. 
For the toy example, we ran all methods for the five common initial points in the literature. 
For CityScapes, NYUv2, and CelebA, we only ran \textsc{ConicGrad}, and the performance measures of all other methods that we compare against are taken from their respective manuscripts. All of our experiments were conducted on a single NVIDIA Tesla V100 32GB GPU.

\subsection{Supervised Learning Experiments}
\label{appdix:sl_exps}


In alignment with the literature~\citep{NEURIPS2023_b2fe1ee8,xiao2024direction,navon2022multi}, 
for CityScapes and NYUv2, we train our method for \(200\) epochs with Adam optimizer~\citep{kingma2014adam} 
and step learning rate scheduler (referred to as StepLR in PyTorch) 
with a decay factor of \(0.5\) every \(100\) epochs. 
At the start of each experiment, the learning rate for \(\theta\) is set to \(0.0001\) and \(0.001\) for CityScapes and NYUv2 respectively. 
The batch size for CityScapes and NYUv2 is \(8\) an d \(2\) respectively, 
and hyperparameters \(c\) and \(\gamma\) is \(0.25, 0.001\) and \(0.75, 0.01\) for CityScapes and NYUv2 respectively. 
Following~\citep{NEURIPS2023_b2fe1ee8}, for CelebA, we train our method for \(15\) epochs with Adam optimizer, 
and there is no scheduler. 
The batch size is \(256\) with a learning rate of \(0.001\), 
and hyperparameters of \textsc{ConicGrad} \(c\) and \(\gamma\) are set to \(0.5, 0.001\). 
We use validation set in CelebA to report best performance, 
and due to lack of validation set on CityScapes and NYUv2, 
we report the average of last \(10\) epochs. 
Further, for all the datasets, we run each experiment for \(3\) random seeds and report the average result.

\paragraph{Performance Error Bars.}
In \cref{tab:nyuv2errorbars} and \cref{tab:cityscapeserrorbars}, 
we report \textsc{ConicGrad}'s mean scores along with their associated standard errors and compare them with FAMO's~\citep{NEURIPS2023_b2fe1ee8} (as reported in their paper). 
\textsc{ConicGrad} not only achieves better overall performance in majority of the metrics, 
but also consistently exhibits lower standard errors. 
This combination of superior performance and reduced variability underscores \textsc{ConicGrad}'s robustness and reliability, 
ensuring consistent results across runs---an essential factor for reproducibility.

\begin{table*}[!h]
\centering
\caption{\textbf{Results on NYUv2 (3 tasks) with Error Bars.} We repeat each experiment over \(3\) different seeds and report the average results (mean) and standard error (stderr).}
\label{tab:nyuv2errorbars}
\scalebox{0.85}{
\begin{tabular}{@{}lcccccccccc@{}}
\toprule
 & \multicolumn{2}{c}{Segmentation} & \multicolumn{2}{c}{Depth} & \multicolumn{5}{c}{Surface Normal} \\ \cmidrule(lr){2-10}
\textbf{Method} & mIoU \(\uparrow\) & Pix Acc \(\uparrow\) & Abs Err \(\downarrow\) & Rel Err \(\downarrow\) & \multicolumn{2}{c}{Angle Dist \(\downarrow\)} & \multicolumn{3}{c}{Within \(t^{\circ}\) \(\uparrow\)} & \textbf{\(\Delta\) m\%} \(\downarrow\) \\
 &  &  &  &  & Mean & Median & 11.25 & 22.5 & 30 &  \\ \midrule
FAMO (mean) & \(38.88\) & \(64.90\) & \(0.5474\) & \(0.2194\) & \(25.06\) & \(19.57\) & \(29.21\) & \(56.61\) & \(68.98\) & \(-4.10\) \\ 
FAMO (stderr) & \(\pm 0.54\) & \(\pm 0.21\) & \(\pm 0.0016\) & \(\pm 0.0026\) & \(\pm 0.06\) & \(\pm 0.09\) & \(\pm 0.17\) & \(\pm 0.19\) & \(\pm 0.14\) & \(\pm 0.39\) \\
\midrule
\textbf{\textsc{ConicGrad}} (mean) & \(38.67\) & \(65.25\) & \(0.5272\) & \(0.2170\) & \(24.70\) & \(19.37\) & \(29.58\) & \(57.09\) & \(69.56\) & \(-5.13\) \\ 
\textbf{\textsc{ConicGrad}} (stderr) & \(\pm 0.39\) & \(\pm 0.21\) & \(\pm 0.0017\) & \(\pm 0.0016\) & \(\pm 0.03\) & \(\pm 0.05\) & \(\pm 0.08\) & \(\pm 0.10\) & \(\pm 0.09\) & \(\pm 0.12\) \\ \bottomrule
\end{tabular}
}
\end{table*}

\begin{table*}[!h]
\centering
\caption{\textbf{Results on CityScapes (2 Tasks) and CelebA (40 Tasks) with Error Bars.} We repeat each experiment over \(3\) different seeds and report the average results (mean) and standard error (stderr).}
\label{tab:cityscapeserrorbars}
\scalebox{0.85}{
\begin{tabular}{@{}lccccccc@{}}
\toprule
 & \multicolumn{5}{c}{\textbf{CityScapes}} && \multicolumn{1}{c}{\textbf{CelebA}} \\ \cmidrule(lr){2-6} \cmidrule(lr){7-8}  
\textbf{Method} & \multicolumn{2}{c}{Segmentation} & \multicolumn{2}{c}{Depth} & \textbf{\(\Delta\) m\%} \(\downarrow\) && \textbf{\(\Delta\) m\%} \(\downarrow\) \\ \cmidrule(lr){2-5}
 & mIoU \(\uparrow\) & Pix Acc \(\uparrow\) & Abs Err \(\downarrow\) & Rel Err \(\downarrow\) &  &  &  \\ \midrule
FAMO (mean) & \(74.54\) & \(93.29\) & \(0.0145\) & \(32.59\) & \(8.13\) & & \(1.21\) \\ 
FAMO (stderr) & \(\pm 0.11\) & \(\pm 0.04\) & \(\pm 0.0009\) & \(\pm 1.06\) & \(\pm 1.98\) & & \(\pm 0.24\) \\
\midrule
\textbf{\textsc{ConicGrad}} (mean) & \(74.22\) & \(93.05\) & \(0.0133\) & \(30.99\) & \(4.53\) & & \(0.10\) \\ 
\textbf{\textsc{ConicGrad}} (stderr) & \(\pm 0.17\) & \(\pm 0.11\) & \(\pm 0.0001\) & \(\pm 0.92\) & \(\pm 0.73\) & & \(\pm 0.44\) \\ \bottomrule
\end{tabular}
}
\end{table*}

\subsection{Reinforcement Learning Experiments}
Our MTRL experiments are based on the MTRL codebase~\citep{Sodhani2021MTRL}, 
following literature~\citep{navon2022multi,NEURIPS2023_b2fe1ee8,xiao2024direction}. 
\textsc{ConicGrad} is trained for \(2\)M (million) steps with a batch size of \(1280\), 
and we evaluate the method every \(30\)k steps. 
The hyperparameters \(c\) and \(\gamma\) are set to \(0.75\) and \(0.01\) respectively. 
We report the best average test performance over \(10\) random seeds. 
The underlying SAC model is trained with Adam optimizer and the learning rate is set to \(0.0003\).

\end{document}